\documentclass[12pt]{article}

\usepackage{enumitem}
\usepackage[page]{appendix}
\usepackage{amssymb}
\usepackage{mathrsfs} 
\usepackage{algorithm}
\usepackage{algorithmic}
\usepackage{hyperref}
\usepackage{comment}
\usepackage{amsthm}
\usepackage{mathtools}
\usepackage{epstopdf,xcolor}
\usepackage{bbm}
\usepackage{dsfont}
\usepackage{pifont}
\usepackage{subcaption}
\usepackage{float}
\usepackage{tikz, subcaption}
\usetikzlibrary{shapes, arrows, positioning}
\usepackage{fullpage}
\usepackage{natbib}
\usepackage{authblk}

\newtheorem{theorem}{Theorem}
\newtheorem{lemma}{Lemma}

\newtheorem{proposition}{Proposition}

\newtheorem{definition}{Definition}

\newtheorem{remark}{Remark}

\newtheorem{assumption}{Assumption}

\newcommand{\indep}{\perp\mkern-9.5mu\perp}

\newcommand{\Q}{\mathcal{Q}}
\newcommand{\E}{\mathbb{E}}

\newcommand{\sH}{\mathcal{H}}

\newcommand{\dotq}{\dot{q}}

\newcommand{\doth}{\dot{h}}

\title{\bf Minimax Kernel Machine Learning for a Class of Doubly Robust Functionals with Application to Proximal Causal Inference}

\author[1]{AmirEmad Ghassami}
\author[2]{Andrew Ying}
\author[1]{Ilya Shpitser}
\author[2]{\\Eric Tchetgen Tchetgen}
\affil[1]{Department of Computer Science, Johns Hopkins University}
\affil[2]{Department of Statistics and Data Science, The Wharton School, University of Pennsylvania}

\date{First Version: April 7, 2021; Current Version: March 7, 2022}

\begin{document}

\maketitle

\begin{abstract}
\cite{robins2008higher} introduced a class of influence functions (IFs) which could be used to obtain doubly robust moment functions for the corresponding parameters. However, that class does not include the IF of parameters for which the nuisance functions are solutions to integral equations. Such parameters are particularly important in the field of causal inference, specifically in the recently proposed proximal causal inference framework of \cite{tchetgen2020introduction}, which allows for estimating the causal effect in the presence of latent confounders. In this paper, we first extend the class of Robins et al. to include doubly robust IFs in which the nuisance functions are solutions to integral equations. Then we demonstrate that the double robustness property of these IFs can be leveraged to construct estimating equations for the nuisance functions, which enables us to solve the integral equations without resorting to parametric models. We frame the estimation of the nuisance functions as a minimax optimization problem. We provide convergence rates for the nuisance functions and conditions required for asymptotic linearity of the estimator of the parameter of interest. The experiment results demonstrate that our proposed methodology leads to robust and high-performance estimators for average causal effect in the proximal causal inference framework.
\end{abstract}

\emph{Keywords:
Doubly Robust Functionals, Minimax Learning, Average Causal Effect, Proximal Causal Inference, Kernel Methods}\\

Accepted to the 25th International Conference on Artificial Intelligence and Statistics\\ (AISTATS) 2022.

\newpage

\section{INTRODUCTION}
\label{sec:intro}

Suppose independent and identically distributed data from a distribution $P$ over variables $V$ are given and we are interested in estimating a finite-dimensional functional of the distribution $\psi_0\coloneqq\psi(P)$. A common way to do so is to use moment functions to construct estimating equations for the parameter of interest. Such moment functions usually depend on other (possibly infinite-dimensional) functions indexing the distribution called nuisance functions. That is, despite the fact that these functions are not of primary interest, they may be needed to obtain an estimator of the target parameter. The concern is that bias in estimating the required nuisance functions can induce excessive bias for the parameter of interest, thus compromising one's ability for accurate inference about $\psi_0$. 

For many applications, one can construct an estimating equation in which the moment function depends on two variation independent nuisance functions. An estimator based on such a moment function is called doubly robust if it is consistent even if one of the nuisance functions is not estimated consistently, provided that the other nuisance function is. That is, double robustness gives the user two chances to estimate the parameter of interest correctly. In addition, in cases where one of the nuisance functions has a slow convergence rate, a parametric ($\sqrt{n}$) convergence rate for the functional of interest can still be obtained if the other nuisance function can be estimated at a fast enough rate \citep{tsiatis2007semiparametric}. The average causal effect is perhaps the most well-studied functional for which under certain conditions a doubly robust moment function exists. For this functional, under the assumption of no unobserved confounders, the nuisance functions are the outcome regression function and the propensity score \citep{hernan2020causal}.

\cite{robins2001comment} and \cite{chernozhukov2016locally} gave conditions for the existence and constructing doubly robust moment conditions in semiparametric models. A common approach for obtaining such a moment function is based on using the influence function (IF) of the parameter of interest.  \cite{robins2008higher} introduced a large class of doubly robust IFs in which nuisance functions are always in the form of a regression function. However, for some functionals of interest, specially in the field of causal inference, the nuisance functions are solutions to complicated integral equations. A prominent example of this case appears in the recently proposed proximal causal inference framework \citep{tchetgen2020introduction, miao2018identifying,cui2020semiparametric}. This framework enables estimating the average causal effect when unobserved confounders are present in the system, yet requires the existence of two conditionally independent sets of proxies of the latent confounders that are sufficiently rich to fulfill certain completeness conditions, also referred to as proximal relevance assumption (e.g., using measurements of biomarkers as proxies of patients underlying biological confounding variables).  
The proximal causal inference framework unifies and connects several existing  causal identification and inference frameworks which leverage  various types of proxies, such as the instrumental variables, negative controls, synthetic controls, and difference-in-differences. 
As an example of a nuisance function which is a solution to an integral equation, as we will explain in detail in Section \ref{sec:apps}, under the assumptions of the proximal causal inference framework, a nuisance function $h$ can be used for identifying the average causal effect of a treatment variable on an outcome variable, where $h$ is a solution to the integral equation in display \eqref{eq:ORproxexist}.

Motivated by proximal causal inference framework, in this paper, we first extend the IF class of \cite{robins2008higher} to include doubly robust IFs in which the nuisance functions are solutions to integral equations (Section \ref{sec:proposal}). In this case, one cannot simply fit a flexible model for the nuisance functions by solving a regression problem. We show that in this case, the doubly robust moment functions, besides their aforementioned desired properties, can also be used for constructing estimating equations for the nuisance functions (Section \ref{sec:minimaxlev}). This enables us to solve the integral equations without resorting to parametric models. The main idea in our nuisance function estimation approach is to estimate each function such that it keeps the expected value of the moment function fixed for the perturbations in the other nuisance function. To implement this idea, we first define a perturbation function which captures the change in the mean of the moment function for a deviation in the nuisance functions. Then we solve a regularized minimax optimization problem which estimates one nuisance function as the function minimizing the perturbation for the worst-case deviation in the other nuisance function. 
The proposed approach yields unbiased estimating equations for each nuisance function that are free of any nuisance function. Moreover, the use of perturbation function elaborates the connection between nuisance function estimation and reduces the sensitivity/bias in the estimation of the functional of interest.

Framing the problem of estimating the nuisance parameters as an optimization problem enables us to use high-performance non-parametric machine learning tools to design our learners. Here we use reproducing kernel Hilbert spaces (RKHS) (Section \ref{sec:kernel}) and present the closed-form solution for the minimax optimization problem for estimating the nuisance functions. We characterize the convergence rate of the nuisance functions based on the recently proposed method of \cite{dikkala2020minimax} that leverages the localized Rademacher complexity of the class of nuisance functions.

Equipped with high-performance estimators for the nuisance functions, we use the cross-fitting approach \citep{bickel1988estimating,van2000asymptotic,robins2008higher,zheng2010asymptotic,chernozhukov2018double} to design an estimator for the parameter of interest (Section \ref{sec:estapproach}). We characterize the measure of ill-posedness of the integral equations that we need to solve, and investigate the complexities caused by ill-posedness of the integral equations and their effect on convergence rate of the parameter of interest. We present the requirements on the convergence rate of the nuisance functions and ill-posedness of the system which guarantees the estimator to be asymptotically linear and root $n$ consistent, i.e., attains parametric convergence rate (Section \ref{sec:analys}). Hence, one can use the influence function of the proposed estimator to obtain confidence intervals for the parameter of interest.
In Section \ref{sec:apps}, we demonstrate how proximal causal inference framework fits in our estimation setup and evaluate our proposed method by estimating average causal effect in the proximal causal inference framework on synthetic data as well as real-data in Section \ref{sec:exp}.

\section{RELATED WORK}
\label{sec:relwork}

There are few other works on using a doubly robust moment function for estimating the nuisance functions. One approach is to estimate the nuisance functions by minimizing the variance of the doubly robust estimator \citep{cao2009improving,tsiatis2011improved,van2010collaborative}. Another perspective is focusing on bias reduction rather than variance reduction and our proposed method also falls into this category \citep{van2014targeted, vermeulen2015bias, avagyan2017honest,cui2019selective}. Especially, \cite{vermeulen2015bias} proposed the bias reduced doubly robust estimation approach, which locally minimizes the squared first-order asymptotic bias of the doubly robust estimator in the direction of the nuisance parameters under misspecification of both working models. \cite{vermeulen2015bias} only consider parametric working models for the nuisance functions; a restriction that we avoid in the present work similar to \cite{robins2008higher} and \cite{cui2019selective}. In fact, our method can be viewed as a proposal for generalization of that work to the non-parametric setup.

There is a growing attention in the literature to the use of machine learning approaches for estimating causal quantities \cite{athey2016approximate,farbmacher2020causal,kallus2018policy,nie2017quasi,shalit2017estimating,wager2018estimation,chernozhukov2019semi,oprescu2019orthogonal,kallus2019localized,dikkala2020minimax,bennett2019deep,hartford2017deep}, with a particular recent interest in minimax machine learning methods \citep{bennett2019deep,dikkala2020minimax,muandet2019dual,liao2020provably,chernozhukov2020adversarial,kallus2021causal}. 
However, it is important to note that in several of these works such as \citep{dikkala2020minimax} and \citep{bennett2019deep}, the target is a dose-response curve. Therefore, the target of estimation is unique. In our work, we use the machine learning tools for estimating functions which are nuisance functions for our parameter of interest, and in general, they do not need to be uniquely identified for the parameter of interest to be uniquely identified. 

Finally, regarding the proximal causal inference framework, existing results mostly rely on parametric assumptions for the working models for the nuisance functions, with the exception of \citep{singh2020kernelnc}, \citep{mastouri2021proximal}, and the independent, concurrent work \citep{kallus2021causal}. 
However, the functional classes, assumptions, and convergence analysis for the parameter of interest in those works are different from ours. 
Specifically, \cite{singh2020kernelnc} does not use the influence function for estimating the causal effect.
\cite{mastouri2021proximal} only considers the so-called proximal g-formula approach for estimation, which is the result mentioned in Theorem \ref{thm:POR} of our paper. Because the influence function is not used in that work, the bias will be first order and therefore, larger than the bias in our work which is not only second order but also product bias. That is, the bias in our method is guaranteed to be of smaller order. In fact, while our estimator is root-$n$ consistent (under the assumptions of Theorem \ref{thm:convergence}), the estimator in that work will generally fail to be root-$n$ consistent due to slow convergence rate of their nonparametric estimator of the bridge function.
\cite{kallus2021causal} consider the use of a minimax learning method in proximal causal inference framework. However, the functional class, assumptions, and convergence analysis for the parameter of interest in that work is different from ours. That work is primarily focused on average causal effect in proximal causal inference, yet we work with a broad functional class and we focus on the double robustness property to derive the estimators and convergence analysis. The functional class considered in \citep{kallus2021causal} is a special case of our class of functionals. This can be seen by including the function $\pi$ in that paper in functions $g_1$ and $g_2$ in our functional. Also, we make explicit use of double-robustness or product-bias property in establishing that our estimator is root-$n$ consistent, regular and asymptotically linear. 
Please see the Supplementary Materials for a detailed description of the related work.

\section{PROPOSED METHOD}
\label{sec:proposal}

Let $V$ denote the variables from which independent and identically distributed data is collected, and let $\psi_0$ be the finite-dimensional parameter of interest. We consider the class of regular parameters $\psi_0$ for which the influence function (IF) is of the form\footnote{See \citep{van2000asymptotic} for the definition of regular parameters and their influence functions.}
\begin{equation}
\label{eq:IF}
\begin{aligned}
IF_{\psi_0}(V)&=q_0(V_q)h_0(V_h)g_1(V)+q_0(V_q)g_2(V)+h_0(V_h)g_3(V)+g_4(V)-\psi_0,
\end{aligned}
\end{equation}
where $h_0(\cdot)$ and $q_0(\cdot)$ are nuisance functions, $V_q$ and $V_h$ are (not necessarily disjoint) subsets of $V$, and $g_1,g_2,g_3$, and $g_4$ are known measurable functions. The IF can be used as a moment function for estimating $\psi_0$, and with slight abuse of notation, we define the moment function
\begin{equation*}
\begin{aligned}
IF(V;\psi,q,h)&\coloneqq q(V_q)h(V_h)g_1(V)+q(V_q)g_2(V)+h(V_h)g_3(V)+g_4(V)-\psi.
\end{aligned}
\end{equation*}
Parameters with IF of the form \eqref{eq:IF} appear in many settings, especially in causal inference and missing data problems. It includes functionals such as average causal effect, expected product of conditional expectations, expected conditional covariance, semiparametric regression, marginal structural mean models, marginal mean under missing at random assumption and certain cases of data missing not-at-random problems, etc. This also applies to the nonparametric instrumental variable problems. For instance, the moment function in  \citep[p. 20]{chernozhukov2016locallyv2} is in the proposed class with the choice of $\gamma$ as $h$ and $\lambda$ as $q$. This class is an extension of the class introduced in \citep{robins2008higher} in that we allow the nuisance functions to have different arguments. That is, in the class of \cite{robins2008higher}, $V_q=V_h$. However, as mentioned in Section \ref{sec:intro}, 
requiring that $V_q=V_h$ will exclude functionals for which the nuisance functions are not solutions to regression problems.
Specially, it will exclude functionals 
in the proximal causal inference framework such as the proximal extensions of the aforementioned examples. Also it will exclude functionals arising in shadow variable approaches to missing data \citep{miao2015identification}.  In \citep{li2021identification}, noting that $\mathbb{E}[h(X,Z)\mid R=1,X,Y]=\gamma(X,Y)$, the influence function given in Equation (16) is in the proposed class with the choice of $\delta$ as $h$ and $\gamma$ as $q$. Our framework provides a unified nonparametric adversarial estimation framework for all these problems. 
 An important example we will use in this paper is the proximal average causal effect, with the IF given in \eqref{eq:proximalIF}.

We require the following assumption on the IF, which is satisfied for a large class of doubly robust IFs, including the one for average causal effect, discussed in Section \ref{sec:apps}.
\begin{assumption}
\label{assumption:varind}
The nuisance functions $q_0$ and $h_0$ are variation independent, that is, they vary over a Cartesian product set $\Q\times\sH$, where we assume $\Q$ and $\sH$ are dense in $\mathcal{L}_2(P_{V_q})$ and $\mathcal{L}_2(P_{V_h})$, respectively. In addition, the (possibly infinite dimensional) parameters governing the marginal distribution of $V_q\cup V_h$ are variation independent of $q_0$ and $h_0$.
\end{assumption}
As an example, in the proximal causal inference framework that we focus on in Section \ref{sec:apps},
one can always choose parameters such that Assumption \ref{assumption:varind} is satisfied. In the integral equation for the proximal setup for estimating the average causal effect, two out of three parameters of the function $h$, the conditional distribution of $W$ given $Z,A,X$, and the conditional distribution of $Y$ given $Z,A,X$ can always be chosen to be variational independent. One can choose to pick the first two, which is the requirement of Assumption \ref{assumption:varind}.

The form of the IF \eqref{eq:IF} provides us with the double robustness property, which states that an estimator based on the moment function $IF(V;\psi,q,h)$ is consistent even if we misspecify one of the nuisance functions (but not both). This property is especially crucial in setups with high-dimensional or non-parametric nuisance functions.  
\begin{proposition}[Double-robustness]
\label{prop:DRHOIF}
For all choices of nuisance functions $h$ and $q$,
\begin{equation}
\label{eq:DR}
\begin{aligned}
\E[IF(V;\psi_0,q_0,h_0)]&=\E[IF(V;\psi_0,q,h_0)]\\
&=\E[IF(V;\psi_0,q_0,h)]=0.
\end{aligned}
\end{equation}
\end{proposition}

\subsection{Leveraging Double-Robustness to Estimate Nuisance Functions}
\label{sec:minimaxlev}
We note that the nuisance functions in IF \eqref{eq:IF} are not functions of the parameter of interest. Therefore, we can first focus on estimating the nuisance functions.  
In this subsection, we establish that the doubly robust moment function, besides its desired properties for estimating the parameter of interest, can also be leveraged for constructing estimating equations for the nuisance functions. 

We define the perturbation at function pair $(q,h)$ towards the pair $(\dotq,\doth)$ as
\begin{equation*}
\begin{aligned}
prt(q,h;\dotq,\doth)\coloneqq IF(V;\psi_0,q\!+\!\dotq,h\!+\!\doth)\!-\!IF(V;\psi_0,q,h).
\end{aligned}
\end{equation*}
Double robustness property in equation \eqref{eq:DR} indicates that at the pair $(q_0,h)$ for true function $q_0$ and any function $h$, the expected value of the perturbation function towards any other function $h+\doth$ should be zero. This motivates the following approach for estimating the nuisance functions: We choose $(\hat{q},\hat{h})$ as the true nuisance function pair if it is the solution to the following estimating equations.
\begin{equation*}
\begin{aligned}
	&\E[prt(\hat{q},h;0,\doth)]=0~~~~~\forall h,\doth,\\
	&\E[prt(q,\hat{h};\dotq,0)]=0~~~~~\forall q,\dotq.
\end{aligned}
\end{equation*}
Using these equations, and leveraging the fact that the IF is linear in each nuisance function, we propose the following optimization-based estimation approach for estimating the nuisance functions.
\begin{equation}
\label{eq:mMest}
\begin{aligned}
(\hat{q},\hat{h})=&
\arg\min_{(q,h)}\big\{
\max_{\doth} \E[ prt(q,h;0,\doth) ]
+\max_{\dotq} \E[ prt(q,h;\dotq,0) ]\big\}
\end{aligned}
\end{equation}
Optimization \eqref{eq:mMest} demonstrates the connection between estimating the nuisance functions and the parameter of interest. Specifically, it shows that finding the nuisance functions indeed requires minimizing the perturbation function, and hence the sensitivity of the parameter of interest to the nuisance functions.

We note that due to Assumption \ref{assumption:varind}, the optimum value of \eqref{eq:mMest} can be obtained by solving two separate optimization problems.
\begin{proposition}
\label{prop:separate}
Under Assumption \ref{assumption:varind}, the solution $(\hat{q},\hat{h})$ to the optimization \eqref{eq:mMest} is obtained by solving
\begin{equation*}
\begin{aligned}	
&\hat{q}=\arg\min_{q}\max_{\doth }
\E\Big[ \doth(V_h)[q(V_q)g_1(V)+g_3(V)] \Big],\\
&\hat{h}=\arg\min_{h}\max_{\dotq }
\E\Big[ \dotq(V_q)[h(V_h)g_1(V)+g_2(V)] \Big],
\end{aligned}
\end{equation*}
That is, $\hat{q}$ and $\hat{h}$ can be estimated using separate optimizations. The solutions $(\hat{q},\hat{h})$ satisfy
\begin{equation}
\label{eq:condmomemnteq}
\begin{aligned}
&\E[\hat{h}(V_h)g_1(V)+g_2(V)|V_q]=0,\\
&\E[\hat{q}(V_q)g_1(V)+g_3(V)|V_h]=0.
\end{aligned}
\end{equation}
\end{proposition}
In the case that $V_q=V_h$, conditional moment equations \eqref{eq:condmomemnteq} suggest that the nuisance functions can be estimated by solving standard regression problems. However when $V_q\neq V_h$, those moment equations will be inverse problems known as Fredholm integral equation of the first kind.

Unfortunately the optimization problems in Proposition \ref{prop:separate} are not stable. 
For instance, for the first optimization,
if for a choice of $q$, $\mathbb{E}[q(V_q)g_1(V)+g_3(V)\mid V_h]$ is not zero, there exists a choice of $\dot{h}$ that can make $\mathbb{E}\big[\dot{h}(V_h)[q(V_q)g_1(V)+g_3(V)]\big]$ arbitrary large, regardless of how close $q$ is to the ground-truth. 
That is, the objective function can be made arbitrarily large for sufficiently rich function space for $\dot{h}$ unless $q$ is evaluated at the truth, and this can happen even in the neighborhood of the truth.
Since in reality we work with finite data, for any choice of $q$ we get a large value for the objective function depending on the function space and data, and we cannot converge to the correct parameter; leading to miss-specification. 
Therefore, we add a squared regularization term for robustness against misspecification and stability of the optimization, as well as improving the convergence rates. Hence, we will instead focus on the following regularized optimization problems.
\begin{equation}
\label{eq:optq}	
\hat{q}=\arg\min_{q}\max_{\doth }
\E\Big[ \doth(V_h)[q(V_q)g_1(V)+g_3(V)]-\doth^2(V_h) \Big],
\end{equation}
\begin{equation}
\label{eq:opth}
\hat{h}=\arg\min_{h}\max_{\dotq }
\E\Big[ \dotq(V_q)[h(V_h)g_1(V)+g_2(V)]-\dotq^2(V_q) \Big].
\end{equation}
In the following result, we show that the proposed penalized  optimizations solve the integral equations \eqref{eq:condmomemnteq}. That is, the stability term does not introduce bias.
\begin{theorem}
\label{thm:conds}
The solutions to the optimization problems \eqref{eq:optq} and \eqref{eq:opth} satisfy the conditional moment equations in \eqref{eq:condmomemnteq}.
\end{theorem}
Proposition \ref{prop:separate} and Theorem \ref{thm:conds} show the connection between perturbing doubly robust functionals and conditional moment function; an observation which to the best of our knowledge is new to this work.

\subsection{Estimation Procedure}
\label{sec:estapproach}
We use cross-fitting estimation approach of \cite{chernozhukov2018double} for separating the estimation of the nuisance functions from the parameter of interest. This approach provides us with the benefit that weaker smoothness requirements are needed for the nuisance functions. In the cross-fitting approach, we partition the samples into $L$ equal size  parts $\{I_1,...,I_L\}$. Consider the finite data version of the estimators from the optimization problems \eqref{eq:optq} and \eqref{eq:opth},  obtained by replacing the population expectation operator $\E[\cdot]$ with the sample average operator $\hat{\E}[\cdot]$. For $\ell\in\{1,...,L\}$, we estimate the nuisance functions $(\hat{q}_\ell,\hat{h}_\ell)$ on data from all parts but $I_\ell$. For all $\ell$, let $\hat{\psi}_\ell$ be the estimation of $\psi_0$ obtained by solving $\frac{1}{|I_\ell|}\sum_{i\in I_\ell}IF(V_i;\hat{\psi}_\ell,\hat{q}_\ell,\hat{h}_\ell)=0$. Our final estimator of $\psi_0$ is obtained by 
\begin{equation}
\label{eq:cf-estimator}
\hat{\psi}=\frac{1}{L}\sum_{\ell=1}^L \hat{\psi}_\ell.
\end{equation}
\subsection{Asymptotic Analysis}
\label{sec:analys}
In this subsection, we study the asymptotic properties of the cross-fitting estimator in display \eqref{eq:cf-estimator}.
We require the following regularity conditions for the results.
\begin{assumption}
\label{assumption:convergence}
\begin{enumerate}[label=(\roman*),leftmargin=*]
\itemsep0em
\item
Functions $g_1$, $g_2$, $g_3$, and $g_4$ are bounded.
\item
There exists a constant $\sigma_1$ such that $\big|\E[g_1(V)|V_q,V_h]\big|>\sigma_1>0$.
\item
Functions $q_0$ and $h_0$ are square-integrable with respect to the underlying  measure $P$.
\item 
$
\min\Big\{	
\sup_{v_q}|\hat{q}_\ell(v_q)|
+\sup_{v_h}|h_0(v_h)|
,
\sup_{v_q}|q_0(v_q)|
+\sup_{v_h}|\hat{h}_\ell(v_h)|
\Big\}<\infty.
$
\end{enumerate}
\end{assumption}
In addition to the regularity conditions, we require  conditions on the convergence of the nuisance functions. In our results, we show that the projected space is the right space to impose the convergence constraints on. Below, we use the conventional notations $\mathbb{P}_n[\cdot]$ and $\text{P}[\cdot]$ to denote empirical and population expectations with respect to the variable $V$. Note that for a function $\hat{f}$, unlike the operator $\E[\hat{f}(V)]$, in $\text{P}[\hat{f}(V)]$, the operator does not take the expectation with respect to the possible randomness in the function $\hat{f}$. For any function $f$, we use the norm notation $\|f\|_2\coloneqq \sqrt{\text{P}[f^2]}$. We need the following definition for our convergence assumption.
\begin{definition}
For any given value $\delta>0$, for given function spaces $\Q$ and $\sH$, let
$
\Q^{|\delta}\coloneqq\big\{ q\in\Q:\big\| \textup{P}[q(V_q)|V_h] \big\|_2\le\delta \big\}$, and
$
\sH^{|\delta}\coloneqq\big\{ h\in\sH:\big\| \textup{P}[h(V_h)|V_q] \big\|_2\le\delta \big\}$.
We define the local measure of ill-posedness as
$
\tau_q(\delta)\coloneqq\sup_{q\in\Q^{|\delta}}\|q\|_2$, and
$\tau_h(\delta)\coloneqq\sup_{h\in\sH^{|\delta}}\|h\|_2$.
\end{definition}
$\tau_q(\delta)$ and $\tau_h(\delta)$ are defined as measures of ill-posedness of the conditional expectation operator. This measure was originally proposed in \citep{chen2012estimation} and is commonly used in the econometrics literature.
The intuition behind this measure is as follows: Projecting functions on a space makes them closer (and hence hard to distinguish) in terms of the $L_2$-norm. The measure of ill-posedness relates a certain $L_2$ distance post-projection, to the corresponding largest $L_2$ distance of the functions before projection (i.e. the extent to which projection shrinks the distance). In other words, how much resolution is lost due to the projection. Therefore, a large ill-posedness measure $\tau_h(\delta)$ implies that $V_q$ does not capture much of the information in $V_h$.

\begin{assumption}
\label{assumption:nuisconv}
$(i)$ $\|\hat{q}_\ell\!-\!q_0\|_2\!=\!o_p(1)$, $\|\hat{h}_\ell\!-\!h_0\|_2\!=\!o_p(1)$, i.e., they converge to zero in probability.

$(ii)$ $
\big\|\textup{P}[g_1(V)\{\hat{q}_\ell(V_q)\!-\!q_0(V_q)\}|V_h]\big\|_2\!\!=\!O(r_q(n))
$,
and 
$
\big\|\textup{P}[g_1(V)\{\hat{h}_\ell(V_h)\!-\!h_0(V_h)\}|V_q]\big\|_2\!=\!O(r_h(n))
$,
where 
$(r_q(n),r_h(n))$ satisfy
\begin{equation*}
\begin{aligned}	
\min\big\{
r_h(n)\tau_q\big(r_q(n)\big) , r_q(n) \tau_h\big(r_h(n)\big)    \big\} = o(n^{-\frac{1}{2}}).
\end{aligned}
\end{equation*}
\end{assumption}
Part $(i)$ of Assumption \ref{assumption:nuisconv} requires the consistency of the estimators of the nuisance functions. In Section \ref{sec:kernel}, we provide an RKHS-based estimation framework, under which, for certain choices of the kernel (such as the Gaussian or Sobolev kernels), this assumption is satisfied.
Part $(ii)$ of Assumption \ref{assumption:nuisconv} requires a certain rate of convergence for the estimated nuisance functions. However, instead of requiring a convergence rate for each nuisance function, we require a rate for their product. Hence, in case that one of them is not converging fast enough the other nuisance function can compensate. This is one of the main attractions of double robust estimators.
In Section \ref{sec:kernel}, we will study the local measure of ill-posedness and convergence rates related to part $(ii)$ for the case that a reproducing kernel Hilbert space is used as the hypothesis class.

We have the following result regarding the convergence of the cross-fitting estimator.
\begin{theorem}
\label{thm:convergence}
Under Assumptions \ref{assumption:convergence} and \ref{assumption:nuisconv}, the estimator $\hat{\psi}$ in  \eqref{eq:cf-estimator} is asymptotically linear and satisfies
\begin{equation*}
\begin{aligned}
\sqrt{n}\{\hat{\psi}-\psi_0\}=\sqrt{n}~\hat{\E}[IF(V;\psi_0,q_0,h_0)]+o_p(1),
\end{aligned}
\end{equation*}
\begin{sloppypar}
\noindent
where $o_p(1)$ demonstrates a sequence of variables which converges to zero in probability, and $\sqrt{n}\{\hat{\psi}-\psi_0\}$ converges in distribution to the Gaussian distribution $\mathcal{N}\big( 0,\text{var}~(IF(V; \psi_0,q_0,h_0)) \big)$.
\end{sloppypar}
\end{theorem}
As a corollary of Theorem \ref{thm:convergence}, one can use the influence function of the proposed estimator to obtain confidence intervals for the parameter of interest.
\begin{remark}
In our approach, we use an IF, which always leads to a second order bias. This fundamental property was recently popularized by \cite{chernozhukov2018double} in terms of Neyman Orthogonality property. However, note that our result is even stronger because not only does our estimator have second order bias, but also it has the so-called mixed bias property, which is the basis of double-robustness and is not guaranteed by Neyman Orthogonality. Please see the Supplementary Materials for a detailed discussion.	
\end{remark}
\section{AVERAGE CAUSAL EFFECT UNDER PROXIMAL CAUSAL INFERENCE FRAMEWORK}
\label{sec:apps}
In this section, we apply the proposed minimax estimation approach to estimating average causal effect from observational data. Let $A$ be a binary treatment variable and $Y$ be the outcome variable. For $a\in\{0,1\}$, let $Y^{(a)}$ be the counterfactual outcome variable representing the outcome if (contrary to the fact) the treatment is set to value $a$. The Average causal effect (ACE) captures the difference in the expected value of the counterfactual outcome variables, that is $\textit{ACE}=\E[Y^{(1)}-Y^{(0)}]$. Therefore, it suffices to focus on the problem of estimating the counterfactual mean of form $\psi_0=\E[Y^{(a)}]$, for $a\in\{0,1\}$. The most popular assumption for studying average causal effect is the so-called \emph{conditional exchangeability} assumption \citep{hernan2020causal}. Intuitively, this assumption requires that we have collected enough covariates of the units in the study that conditional on those covariates, the observational data is effectively as good as a conditionally randomized experiment. That is, conditional on the covariates, the units in the treated and untreated groups are exchangeable.

Despite its popularity, the exchangeability assumption is often violated even in laboratory settings as there may be confounders of the treatment and the outcome which are latent.  Hence, causal inference approaches are needed that are capable of handling unobserved confounders. Proximal causal inference is one such approach in which, although latent confounders are allowed, presence of certain proxy variables for the confounders is required. Formally, the proximal causal inference framework allows for unobserved confounder variable (possibly vector-valued) $U$, observed confounder $X$, and requires access to  proxy variables $Z$ and $W$ which satisfy the following conditions.
\begin{assumption}
\label{assumption:prox1}
~
\begin{itemize}[leftmargin=*]
\itemsep0em
\item $Y^{(a,z)}=Y^{(a)}$ almost surely, for all $a$ and $z$.
\item $W^{(a,z)}=W$ almost surely, for all $a$ and $z$.
\item $(Y^{(a)},W)\indep (A,Z)|(U,X)$, for $a\in\{0,1\}$.
\end{itemize}
\end{assumption}
\begin{figure}[t]
\centering
\tikzstyle{block} = [circle, inner sep=2.5pt, fill=lightgray]
\tikzstyle{input} = [coordinate]
\tikzstyle{output} = [coordinate]
\begin{tikzpicture}
\tikzset{edge/.style = {->,> = latex'}}
            \node[block] (u) at  (0,1.8) {$U$};
            \node[] (x) at  (0,.9) {$X$};
            \node[] (z) at  (-2,.9) {$Z$};
            \node[] (w) at  (2,.9) {$W$};
            \node[] (a) at  (-1.1,0) {$A$};
            \node[] (y) at  (1.1,0) {$Y$};
			\draw[edge] (u) to (x);
			\draw[edge] (u) to (z);
			\draw[edge] (u) to (w);
			\draw[edge] (u) to (a);
			\draw[edge] (u) to (y);
			\draw[edge] (x) to (z);
			\draw[edge] (x) to (w);
			\draw[edge] (x) to (a);
			\draw[edge] (x) to (y);						
			\draw[edge] (z) to (a);			
            \draw[edge] (w) to (y);
            \draw[edge] (a) to (y);
\end{tikzpicture}
\caption{A proximal DAG.}
\label{fig:DAG}
\end{figure}
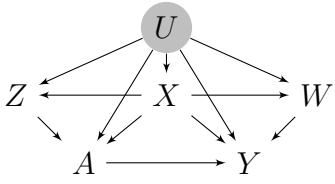

Note that Assumption \ref{assumption:prox1} implies $Y\indep Z|A,U,X$ and $W\indep (A,Z)|U,X$, which can be taken as primitive identification conditions if one does not wish to entertain an intervention on $Z$. 
Figure \ref{fig:DAG} depicts an example graphical representation that satisfies these assumptions (the gray variable $U$ is unobserved). 

In order to obtain identifiability, \cite{miao2018identifying} considered the following assumptions.
\begin{assumption}[Consistency and Positivity]
\label{assumption:pos2}
$(i)$ $Y^{(A,Z)}=Y$,
$(ii)$ $0<Pr(A=1|U,X)<1$ a.s.
\end{assumption}
\begin{assumption}
\label{assumption:compexist1}
$(i)$ For any $a$, $x$, if $\E[g(U)|Z,A = a,X = x]=0$ almost surely, then $g(U)=0$ almost surely.
$(ii)$ There exists an outcome confounding bridge function $h_0(w,a,x)$ that solves the following integral equation
\begin{equation}
\label{eq:ORproxexist}
\begin{aligned}	
\E[Y|Z,A,X]=\E[h_0(W,A,X)|Z,A,X].
\end{aligned}
\end{equation}
\end{assumption}
\cite{miao2018identifying} established the following nonparametric identification result.
\begin{theorem}[\cite{miao2018identifying}]
\label{thm:POR}
Under Assumptions \ref{assumption:prox1}, \ref{assumption:pos2}, and \ref{assumption:compexist1},	the counterfactual mean $\E[Y^{(a)}]$ is nonparametrically identified by
$\E[Y^{(a)}]= \E[h_0(W,a,X)]$.
\end{theorem}
\cite{cui2020semiparametric} proposed an alternative proximal identification result based on the following counterpart of Assumption \ref{assumption:compexist1}.
\begin{assumption}
\label{assumption:compexist2}
$(i)$ For any $a$, $x$, if $\E[g(U)|W,A = a,X = x]=0$ almost surely, then $g(U)=0$ almost surely.
$(ii)$ There exists a treatment confounding bridge function $q_0(z,a,x)$ that solves the following integral equation
\begin{equation}
\label{eq:IPWproxexist}
\begin{aligned}	
\E[q_0(Z,a,X)|W,A=a,X]=\frac{1}{P(A=a|W,X)}.
\end{aligned}
\end{equation}
\end{assumption}
\cite{cui2020semiparametric} established the following nonparametric identification result.
\begin{theorem}[\cite{cui2020semiparametric}]
\label{thm:PIPW}
Under Assumptions \ref{assumption:prox1}, \ref{assumption:pos2}, and \ref{assumption:compexist2}, the counterfactual mean $\E[Y^{(a)}]$ is nonparametrically identified by
$\E[Y^{(a)}]= \E[I(A=a)q_0(Z,a,X)Y]$.
\end{theorem}
\begin{remark}
Part $(i)$ in Assumption \ref{assumption:compexist1} and \ref{assumption:compexist2} is known as completeness condition; a technical condition central to the study of sufficiency in statistical inference. 
Regarding part $(ii)$ of these assumptions, note that equations \eqref{eq:ORproxexist} and \eqref{eq:IPWproxexist} define inverse problems known as Fredholm integral equations of the first kind. 
A sufficient condition for part $(ii)$ of Assumption \ref{assumption:compexist1} is the following completeness assumption together with certain mild regularity conditions \citep{miao2018confounding}:
For any $a$, $x$, if $\E[g(Z)|W,A = a,X = x]=0$ almost surely, then $g(Z)=0$ a.s.
Similar result holds for part $(ii)$ of Assumption \ref{assumption:compexist2}.
\end{remark}

Let $V=(X,Z,W,A,Y)$.
\cite{cui2020semiparametric} derived the locally semiparametric efficient influence function for the parameter of interest $\psi_0$ in a nonparametric model in which both confounding bridge functions are unrestricted, at the intersection submodel where both are uniquely identified as
\begin{equation}
\label{eq:proximalIF}
\begin{aligned}	
&IF_{\psi_0}(V)=-I(A=a)q_0(Z,A,X)h_0(W,A,X)\\
&+I(A=a)Yq_0(Z,A,X)+h_0(W,a,X)-\psi_0,
\end{aligned}
\end{equation}
This influence function satisfies our generic doubly robust functional form in expression \eqref{eq:IF} with the choice of 
 $g_1(V)=-I(A=a)$, $g_2(V)=I(A=a)Y$, $g_3(V)=1$, and 
$g_4(V)=0$.
Therefore, we can apply Theorem \ref{thm:conds} to obtain the following result.
\begin{proposition}
\label{prop:proxATE}
For arm $A\!\!=\!a$, let $(\tilde{q}_a,\tilde{h}_a)$ be a solution to the population-level minimax optimizations
\begin{equation*}
\begin{aligned}	
\tilde{h}_a&=\arg\min_h\max_q \E[\{-I(A=a)h(W,X)+I(A=a)Y\}q(Z,X)-q^2(Z,X)],\\
\tilde{q}_a&=\arg\min_q\max_h \E[\{-I(A=a)q(Z,X)+1\}h(W,X)-h^2(W,X)].
\end{aligned}
\end{equation*}
Then functions $\tilde{q}(z,a,x)=I(a=0)\tilde{q}_0(z,x)+I(a=1)\tilde{q}_1(z,x)$, and 
$\tilde{h}(w,a,x)=I(a=0)\tilde{h}_0(w,x)+I(a=1)\tilde{h}_1(w,x)$ solve the integral equations \eqref{eq:ORproxexist} and \eqref{eq:IPWproxexist}, respectively.
\end{proposition}
Note that surprisingly, the propensity score does not appear in the moment equations in Proposition \ref{prop:proxATE}; a fact previously noted in \citep{cui2020semiparametric}, although motivated from a different perspective.

\section{DOUBLY ROBUST KERNEL ESTIMATOR}
\label{sec:kernel}

In this section, we describe solving the minimax problems \eqref{eq:optq} and \eqref{eq:opth} when an RKHS is used as the hypothesis class. We demonstrate that in this case, fast convergence rates can be obtained and we present a closed-form solution for the optimization problem which renders an easy implementation of the method possible.
We further provide results regarding estimating the ill-posedness measure for RKHSes.

Consider the finite data version of the estimators by replacing the population expectation operator $\E[\cdot]$ with sample average operator $\hat{\E}[\cdot]$. We propose the following Tikhonov regularization-based optimizations:
\begin{equation}
\label{eq:optkernelq}
\begin{aligned}
\hat{q}=\arg\min_{q\in\mathcal{Q}}&\sup_{h\in\mathcal{H}} \hat{\E}\big[ h(V_h)[q(V_q)g_1(V)+g_3(V)]
- h^2(V_h) \big]-\lambda^{q}_{\sH}\|h\|_{\mathcal{H}}^2+\lambda^{q}_{\Q}\|q\|_{\mathcal{Q}}^2,
\end{aligned}
\end{equation}
\begin{equation}
\label{eq:optkernelh}
\begin{aligned}
\hat{h}=\arg\min_{h\in\mathcal{H}}&\sup_{q\in\mathcal{Q}} \hat{\E}\big[ q(V_q)[h(V_h)g_1(V)+g_2(V)]
- q^2(V_q) \big]-\lambda^{h}_{\Q}\|q\|_{\mathcal{Q}}^2+\lambda^{h}_{\sH}\|h\|_{\mathcal{H}}^2,
\end{aligned}
\end{equation}
where we assume $\mathcal{H}$ and $\mathcal{Q}$ are RKHSes with kernels $K_{\mathcal{H}}$ and $K_{\mathcal{Q}}$, respectively, equipped with the RKHS norms 
$\|\cdot\|_{\mathcal{H}}$ 
and 
$\|\cdot\|_{\mathcal{Q}}$.
Since the form of the two optimization problems are the same, in the following, we only discuss optimization problem \eqref{eq:optkernelh}. We drop the superscripts $h$ from $\lambda^{h}_{\Q}$ and $\lambda^{h}_{\sH}$ to make the notations less cluttered, and denote sample size by $n$.
\subsection{Closed-Form Solution for the Minimax Optimization}
We start by showing that the optimization problem in Equation \eqref{eq:optkernelh} has a closed-form solution.
Define
$K_{\Q,n}=(K_\Q(V_{q_i},V_{q_j}))_{i,j=1}^{n}$
and $K_{\sH,n}=(K_\sH(V_{h_i},V_{h_j}))_{i,j=1}^{n}$
as the empirical kernel matrices corresponding to spaces $\Q$ and $\sH$, respectively.

\begin{proposition}
\label{prop:mingen}
Equation \eqref{eq:optkernelh} achieves its optimum at $\hat{h}=\sum_{i=1}^n\alpha_{i}K_{\sH}(V_{h_i},\cdot)$,
with $\alpha$ defined as 
\begin{equation*}
\begin{aligned}	
\alpha=&-\big( K_{\sH,n}diag(g_{1,n}) \Gamma diag(g_{1,n}) K_{\sH,n}+n^2\lambda_\sH K_{\sH,n}\big)^\dagger
 K_{\sH,n}diag(g_{1,n}) \Gamma g_{2,n},
\end{aligned}
\end{equation*}
where 
$\Gamma=\frac{1}{4}K_{\Q,n}( \frac{1}{n} K_{\Q,n}+\lambda_\Q I_n)^{-1}$,
$diag(g_{1,n})$ is a diagonal matrix with $g_1(V_i)$ as the $i$-th diagonal entry, and $g_{2,n}\coloneqq[g_2(V_1)\cdots g_2(V_n)]^\top$,
and $^\dagger$ denotes Moore-Penrose pseudoinverse.
\end{proposition}

\subsection{Convergence Analysis of the Nuisance Functions}
In Section \ref{sec:proposal}, we observed that the solution to the population level optimization problem \eqref{eq:opth} satisfies the conditional moment equation $\E[\hat{h}(V_h)g_1(V)+g_2(V)|V_q]=0$, which is equivalent to $\E\big[g_1(V)\{\hat{h}(V_h)-{h}_0(V_h)\}\big|V_q\big]=0$. 
Hence, to quantify the performance of an estimator $\hat{h}$, we define the \emph{projected risk} of the estimator as
\begin{equation*}
\begin{aligned}	
R(\hat{h})\coloneqq \big\|\text{P}\big[g_1(V)\{\hat{h}(V_h)-{h}_0(V_h)\}\big|V_q\big]\big\|_2.
\end{aligned}
\end{equation*}
We will use the recently proposed approach of \cite{dikkala2020minimax} for bounding the projected risk of the regularized minimax estimator in \eqref{eq:optkernelh}. 
The bound that we provide is based on the critical radii of the involved function classes, which are defined by upper bounding the localized Rademacher complexity of the class. See the Supplementary Materials for the definition of localized Rademacher complexity, critical radius, as well as a computation method when an RKHS is used as the function class.

For any function class $\mathcal{F}$, let $\mathcal{F}_B\coloneqq\{f\in\mathcal{F}:\|f\|^2_{\mathcal{F}}\le B\}$.
We require the following conditions on the function classes $\sH$ and $\Q$ for the convergence rate results.
\begin{assumption}
\label{assumption:convergenceN}
$(i)$
$\sH$ and $\Q$ are classes of $b$-uniformly bounded functions, where without loss of generality, we assume $b=1$. The functions $g_1$ and $g_2$ are also bounded.
$(ii)$
$\Q$ is a symmetric class, i.e., if $q\in\Q$, then $-q\in\Q$. Also, $\Q$ is a star-convex class, i.e., if $q\in\Q$, then $\alpha q\in\Q$, for all $\alpha\in[0,1]$.
$(iii)$
$h_0\in\sH$ with $B$ an upper bound on $\|h_0\|^2_{\sH}$.
$(iv)$
For all $h\in\sH$, $\textup{P}\big[g_1(V)\{{h}(V_h)-{h}_0(V_h)\}\big|V_q\big]\in\Q_{L^2\|h-h_0\|^2_\sH}$.
Specifically, with $U=2L^2B$, for all $h\in\sH_B$, we have $\textup{P}\big[g_1(V)\{{h}(V_h)-{h}_0(V_h)\}\big|V_q\big]\in\Q_U$.
\end{assumption}

\begin{theorem}
\label{thm:nuisanceconv}
Define the function class $\mathcal{G}_B\coloneqq\big\{ 
V\rightarrow\alpha g_1(V)\{h(V_h)-h_0(V_h)\}\textup{P}\big[g_1(V)\{h(V_h)-h_0(V_h)\}\big|V_q\big],
\text{where }\alpha\in[0,1], h\in\sH,h-h_0\in\sH_B
\big\}$.
Let $\delta_n$ be an upper bound on the critical radii of $\Q_{U}$ and $\mathcal{G}_B$, and define
$\delta\coloneqq\delta_n+c_0\sqrt{\frac{\log(c_1/\zeta)}{n}}$, for some constants $c_0$ and $c_1$. If $\lambda_\Q\ge\frac{\delta^2}{U}$ and $\lambda_{\sH}\ge\lambda_{\Q}L^2c_2$, for some constant $c_2$, then under Assumption \ref{assumption:convergenceN}, with probability $1-3\zeta$, the estimator $\hat{h}$ in \eqref{eq:optkernelh} satisfies
\begin{equation*}
\begin{aligned}	
R(\hat{h})=
O\Big(\delta+\frac{\lambda_\sH}{\delta}\|h_0\|^2_\sH\Big).
\end{aligned}
\end{equation*}
\end{theorem}

\subsection{Characterization of Ill-posedness Measure}
In Section \ref{sec:proposal}, for any given value $\delta>0$, we defined the local measure of ill-posedness. Due to the Tikhonov regularization used in our minimax estimator, assume that there exists a constant $C$, such that the output of the estimator satisfies $\hat{h}^*-h_0\in\sH_C$. Therefore, we define $\sH_C^{|\delta}\coloneqq\big\{ h\in\sH_C:\big\| \textup{P}[h(V_h)|V_q] \big\|_2\le\delta \big\}$, and focus on the measure of ill-posedness $\tau_h(\delta)\coloneqq\sup_{h\in\sH_C^{|\delta}}\|h\|_2$.

If $\sH$ is an RKHS, $\tau_h(\delta)$ can be bounded using the eigenvalues of the RKHS and a measure of dependency between $V_q$ and $V_h$. This also implies that convergence of the projected risk to zero leads to vanishing of the RMSE, i.e., as required in Assumption \ref{assumption:nuisconv}, the estimator is consistent.
Let $\{\mu_j\}_{j=1}^\infty$ and $\{\phi_j\}_{j=1}^\infty$ be the eigenvalues and eigenfunctions of the RKHS $\sH$. For any $m\in\mathbb{N}_+$, let $V_m$ be the $m\times m$ matrix with entry $(i,j)$ defines as $[V_m]_{i,j}=\E\big[ \E[\phi_i(V_h)|V_q]\E[\phi_j(V_h)|V_q] \big]$.
Note that the quantity on the right hand side of the equality measures how much the eigenfucntions are smoothened by the conditional expectation operator. For instance, if $V_h=V_q$, as is in the case of ACE under no unmeasured confounders, then $V_m$ will be the identity matrix of size $m$. Similarly, for the proximal setup, if the proxy variables are very informative about the unobserved confounder, $V_h$ and $V_q$ will be highly correlated, and the matrix $V_m$ will be close to identity matrix.

\begin{lemma}[\cite{dikkala2020minimax}]
\label{lem:ill}
Let $\lambda_m$ be the minimum eigenvalue of $V_m$, and suppose that for all $i\le m<j$, $\Big|\E\big[ \E[\phi_i(V_h)|V_q]\E[\phi_j(V_h)|V_q] \big]\Big|\le c\lambda_m$, for some constant $c$. 
Then
\begin{equation*}
\begin{aligned}	
\tau_h(\delta)^2\le\min_{m\in\mathbb{N}_+}\Big\{ \frac{4\delta^2}{\lambda_m}+(4c^2+1)C\mu_{m+1} \Big\}.
\end{aligned}
\end{equation*}
\end{lemma}
As an example of the application of Lemma \ref{lem:ill}, suppose we use an RKHS, such as a Sobolev RKHS, with polynomial eigendecay, i.e., $\mu_m\sim m^{-a}$, and the smoothing effect of the conditional expectation operator is also polynomial, i.e., $\lambda_m\sim m^{-b}$, which implies that $\tau_h(\delta)=O(\delta^{\frac{a}{a+b}})$. By Theorem \ref{thm:convergence}, we need $\sqrt{n}r_h(n)r_q(n)^{\frac{a}{a+b}}=o(1)$.
Therefore, if $r_h(n)=n^{-r_h}$ and $r_q(n)=n^{-r_q}$, we require $\frac{1}{2}<r_h+\frac{ar_q}{a+b}$ for $\sqrt{n}$-convergence of the functional of interest.

\begin{figure*}[t!]
    \centering
    \includegraphics[scale = 0.7]{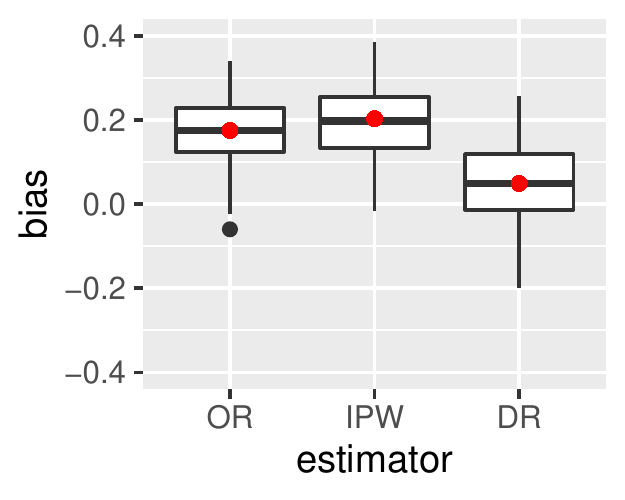}
    \includegraphics[scale = 0.7]{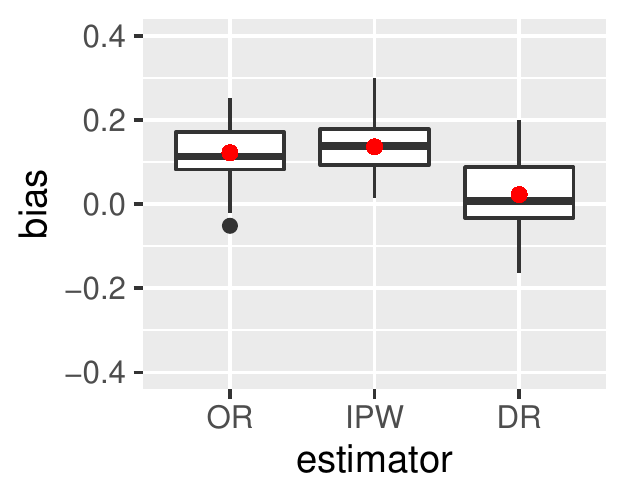}
    \includegraphics[scale = 0.7]{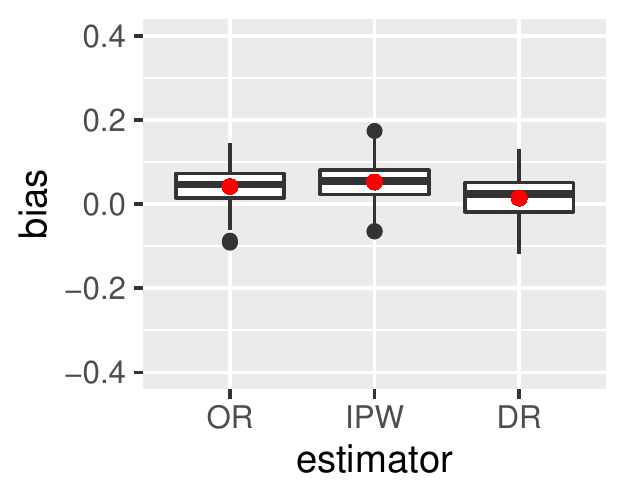}
    \caption{Boxplots of simulation results when $n = 800, 1600, 3200$, respectively.
    }
    \label{fig:syntheticexp}
\end{figure*}

\section{EXPERIMENTS}
\label{sec:exp}

{\bf Synthetic-Data Experiments.} 
We investigate operational characteristics of our framework by applying doubly robust kernel estimators developed in Section \ref{sec:kernel} to estimate $h_0$ and $q_0$ in Section \ref{sec:apps} in a range of simulation settings.\footnote{The implementation is publicly available at \href{https://github.com/andrewyyp/Kernel-Doubly-Robust}{https://github.com/andrewyyp/Kernel-Doubly-Robust}.} 
We generate the data such that the assumptions in Section \ref{sec:apps} hold. We consider sample size $n \in \{800, 1600, 3200\}$ and repeat each simulation for 100 times. The details of the data generating mechanisms as well as additional simulation results considering different data generating mechanisms are provided in the Supplementary Materials.
We estimate the average causal effect using 5 folds cross fitting, with 5 folds cross validation to tune the hyperparameters and kernel bandwidths when learning $h_0$ and $q_0$. We then compare the performance of estimation of ACE obtained from the following three approaches: 
$(i)$ Proximal outcome regression (POR) estimator, which is based on the result in Theorem \ref{thm:POR},
$(ii)$ Proximal inverse probability weighting (PIPW) estimator, which is based on the result in Theorem \ref{thm:PIPW}, and
$(iii)$ Proximal doubly robust (PDR) estimator, which is based on the IF in display \eqref{eq:proximalIF}.
The results of the comparisons are shown in Figure \ref{fig:syntheticexp}. All three estimators attain smaller bias as sample size becomes larger. As expected, the PDR estimator outperforms POR and PIPW estimators in all cases.

{\bf Real-Data Analysis.} 
We reanalyze the Study to Understand Prognoses and Preferences for Outcomes and Risks of Treatments (SUPPORT) with the aim of evaluating the causal effect of right heart catheterization (RHC) during the initial care of critically ill patients in the intensive care unit (ICU) on survival time up to 30 days \citep{connors1996effectiveness}. The same dataset has been analyzed using proximal framework in \citep{tchetgen2020introduction, cui2020semiparametric} with parametric estimation on nuisance parameters.
Data are available on 5735 individuals, 2184 treated and 3551 controls. In total, 3817 patients survived and 1918 died within 30 days. The outcome Y is the number of days between admission and death or censoring at day 30. We include 11 baseline covariates to adjust for potential confounding.  
We kept choices of $Z$ and $W$ the same as in \citep{tchetgen2020introduction, cui2020semiparametric}.  
Details on baseline covariates, hyperparameters and bandwidths are given in the Supplementary Materials. The results are summarized in Table \ref{tab:applicationresults}.

\begin{table}[t!]
\caption{Causal effect estimates (standard deviations) and 95\% confidence intervals.}
\label{tab:applicationresults}
\centering
\begin{tabular}{|lll|}
\hline
& ACE (SDs) & 95\% CIs \\ \hline
$\hat \psi_{\text{POR}}$  &  -1.76 (0.24)  & (-2.23, -1.29)\\
$\hat \psi_{\text{PIPW}}$ &  -1.59 (0.24)  & (-2.07, -1.11)\\
$\hat \psi_{\text{PDR}}$  &  -1.70 (0.24)  & (-2.17, -1.22)\\ \hline        
\end{tabular}
\end{table}

Concordance between the three proximal estimators offers confidence in modeling assumptions, indicating that RHC may have a more harmful effect on 30 day survival among critically ill patients admitted into an intensive care unit than previously reported. In closing, our proximal estimates are well aligned with results obtained in \citep{tchetgen2020introduction, cui2020semiparametric}.

\section{CONCLUSION}
\label{sec:conc}

We have demonstrated that the double-robustness property of a moment function can be used for constructing estimating equations for the nuisance components of the moment function. We framed the idea in terms of a minimax optimization approach, in which we choose a nuisance function such that it minimizes the perturbation of the expected value of the moment function for the worst-case deviation in the other nuisance function. Framing the problem in terms of an optimization, enabled us to use tools from machine learning for implementing the idea and techniques from statistical learning theory to analyze the convergence rates. Specifically, we used kernel methods for designing the learners and Rademacher complexity analysis for obtaining the convergence rates. 
As is the case with most of nonparametric learners, the approach introduces challenges such as choosing the hyper parameters and having larger time complexity. However, it gives us robustness with respect to model misspecification compared to parametric learners.
We provided conditions on the product error of the nuisance functions, as well as the ill-posedness of a conditional expectation operator to obtain the property of asymptotic linearity for the estimator of the parameter of interest. 
We investigated the application of the developed methodology in estimating the average causal effect in the proximal causal inference framework. The experiment results confirmed the superiority of the double robust learner compared to methods that only use one part of the distribution.

\subsubsection*{Acknowledgements}
\begin{sloppypar}
Eric J Tchetgen Tchetgen was supported by NIH grants R01AI27271, R01CA222147, R01AG065276, R01GM139926.
Ilya Shpitser was supported by grants: NSF CAREER 1942239, ONR N00014-21-1-2820, NIH R01 AI127271-01A1, and NSF 2040804.
\end{sloppypar}

\newpage

\bibliography{Refs.bib}
\bibliographystyle{apalike}


\newpage 

\begin{center}
    {\LARGE \bf Supplementary Materials}
\end{center}

\vspace{20mm}

\appendix

\section{RELATED WORK}

In recent years, doubly robust moment functions have been proposed for several different parameters of interest \citep{scharfstein1999adjusting,robins2001comment,bang2005doubly,kang2007demystifying,vansteelandt2008multiply,rotnitzky2012improved,van2011targeted,okui2012doubly,tchetgen2012semiparametric,chernozhukov2016locally,chernozhukov2018double}. The predominant approach for constructing doubly robust moment functions is by using the influence function of the parameter of interest. The seminal work of \cite{scharfstein1999adjusting} took this approach for the study of estimating the average outcome in the presence of missing data. \cite{chernozhukov2016locally} derived several novel classes of doubly robust moment conditions by adding to identifying moment functions the nonparametric influence functions, and discussed sufficient conditions for the existence of doubly robust estimating functions. In our work, we study a general class of doubly robust influence functions as our moment function, which is an extension of a class introduced in \citep{robins2008higher}.

There are few other works on using a doubly robust moment function for estimating the nuisance functions. One approach is to estimate the nuisance functions by minimizing the variance of the doubly robust estimator \citep{cao2009improving,tsiatis2011improved,van2010collaborative}. Another perspective is focusing on bias reduction rather than variance reduction and our proposed method also falls into this category \citep{van2014targeted, vermeulen2015bias, avagyan2017honest,cui2019selective}. Especially, \cite{vermeulen2015bias} proposed the bias reduced doubly robust estimation approach, which locally minimizes the squared first-order asymptotic bias of the doubly robust estimator in the direction of the nuisance parameters under misspecification of both working models. \cite{vermeulen2015bias} only considers parametric working models and the estimating equations are based on the derivative of the functions. There is no straight forward generalization of their approach in the literature, and stability concerns in terms of framing that approach as optimization problem can arise. In fact, our method can be viewed as a proposal for generalization of that work to the non-parametric setup.

Our approach is related to the series of work \citep{blundell2007semi,chen2012estimation,chen2018optimal} in which functionals of non-parametric instrumental variable response curve are considered. Specifically, \cite{chen2012estimation} proposed a class of penalized sieve minimum distance estimators for this task. Our work provides analogous techniques for a large class of doubly robust functionals using minimax estimation of nuisance functions with the hypothesis class chosen to be an RKHS.

There is a growing attention in the literature to the use of machine learning approaches for estimating causal quantities \citep{athey2016approximate,farbmacher2020causal,kallus2018policy,nie2017quasi,shalit2017estimating,wager2018estimation,chernozhukov2019semi,oprescu2019orthogonal,kallus2019localized,dikkala2020minimax,bennett2019deep,hartford2017deep}, with a particular recent interest in minimax machine learning methods \citep{bennett2019deep,dikkala2020minimax,muandet2019dual,liao2020provably,chernozhukov2020adversarial,kallus2021causal}. 
It is important to note that in works such as \citep{dikkala2020minimax} and \citep{bennett2019deep}, the target is a dose-response curve. Therefore, the target of estimation is unique. In our work, functions $h$ and $q$ are just nuisance functions for our parameter of interest, and in general, they do not need to be uniquely identified for the parameter of interest to be uniquely identified. For the convergence analysis of the kernel estimator, we used the results of \cite{dikkala2020minimax}. However, their loss function is a special case of the loss functions resulted from our function class. Also, we utilized the universality property of RKHSes and tailored and streamlined the proof and arguments to account for this property. Also, we modified the presentation form of the hyper parameters.
Also, the starting point of our approach is a doubly robust functional not a conditional moment function as in \citep{dikkala2020minimax}, and the connection between the former and the latter was not stated in the literature before and is one of the contributions of our work. In our work, we start from perturbing a doubly robust functional and prove that this can lead to a conditional moment function.

The authors in the concurrent work \citep{chernozhukov2020adversarial} also use results of \cite{dikkala2020minimax} for function estimation, yet they do not use the double-robustness property for estimating nuisance function and their approach is for a different functional class. The functional class considered in that work does not includes the influence function of the proximal framework. Therefore, we believe that class is not suited for the goal of our paper. Also, that work does not leverage the mixed bias property which is our main tool for estimation. In general, their functional class is restricted to those that do not require ill-posedness concerns regarding projection; both hypothesis spaces are on functions with same arguments.

Finally, regarding the proximal causal inference framework, there is a growing literature on using proxy variables for causal inference from observational data. Extensive discussion of proxy variables encountered in health and social sciences can be found in \citep{lipsitch2010negative, kuroki2014measurement,miao2018identifying,shi2020multiply,shi2020selective,sofer2016negative}. 
The proximal causal inference framework has been recently also extended to other setups such as mediation analysis \citep{dukes2021proximal, ghassami2021proximal}, data fusion \citep{ghassami2022combining, imbens2022long}, and longitudinal data analysis \citep{ying2021proximal, shi2021theory, imbens2021controlling}.
The specific model that we use in our paper is developed in \citep{miao2018identifying,miao2018confounding,tchetgen2020introduction}. Specifically, \cite{miao2018identifying} established sufficient conditions for nonparametric identification of causal effects using a pair of proxies. Their setup was further studied in \citep{vlassis2020proximal} and \citep{deaner2021proxy}. Also, \cite{cui2020semiparametric} developed an efficient semiparametric estimator for the proximal framework. 
\cite{cui2020semiparametric} also proposed estimating equations for the nuisance functions. That work considers a parametric working model for the nuisance functions. Then considers the parameters of the nuisance functions as the parameter of interest and derives the influence functions for these parameters in order to generate estimating equations for the latter. This leads to unconditional moment equations that may be used to estimate parametric models for nuisance bridge functions, yet under certain considerations, as we establish in this paper, these moments equations can generate conditional moment equations allowing for nonparametric estimation of bridge functions. We emphasize that our nonparametric conditional moment equations are generated using a completely different, and more general argument than that of Cui et al. as it applies to a broader class of functionals. 

Existing results for proximal causal inference mostly rely on parametric assumptions for the working models for the nuisance functions, with the following exceptions: \citep{singh2020kernelnc} which uses RKHSes, however, does not use the influence function for estimating the causal effect.
\cite{mastouri2021proximal} only considers the so-called proximal g-formula approach for estimation, which is the result mentioned in Theorem \ref{thm:POR} of our paper. Because the influence function is not used in that work, the bias will be first order and therefore, larger than the bias in our work which is not only second order but also product bias. That is, the bias in our method is guaranteed to be of smaller order. In fact, while our estimator is root-$n$ consistent (under the assumptions of Theorem \ref{thm:convergence}), the estimator in that work will generally fail to be root-$n$ consistent due to slow convergence rate of their nonparametric estimator of the bridge function.
Also, in an independent, concurrent work, \cite{kallus2021causal} consider the use of a minimax learning method in proximal causal inference framework. However, the functional class, assumptions, and convergence analysis for the parameter of interest in that work is different from ours. 
The convergence analysis for the stabilized minimax optimization in that work is also based on the results of \cite{dikkala2020minimax}.
However, that work is primarily focused on ACE in proximal causal inference, yet we work with a broad functional class and we focus on the double robustness property to derive the estimators and convergence analysis. The functional class considered in \citep{kallus2021causal} is a special case of our class of functionals. This can be seen by including the function $\pi$ in that paper in functions $g_1$ and $g_2$ in our functional.

\subsection{Connection to the Double/Debiased Machine Learning Framework and Neyman Orthogonality}

Our results are formally grounded in semiparametric theory. We use a first order influence function, which always leads to a second order bias \citep{bickel1993efficient,newey1990semiparametric, robins2017minimax}. This fundamental property was recently popularized by \cite{chernozhukov2018double} in terms of Neyman orthogonality property. However, note that our result is even stronger because not only our estimator has second order bias, but also it has the so-called mixed bias property, which is the basis of double-robustness. That is, simply having Neyman orthogonality property does not necessarily lead to double-robustness, which is the main feature of our function class, and hence the results of Chernozhokov et al. are not enough to derive ours. Also, due to the same fact, our proof techniques are different. Neyman orthogonality is not the starting assumption in our theorem even though, again, by virtue of the fact that we are using an influence function, our functional does satisfy the Neyman orthogonality property. 

Neyman Orthogonality can be interpreted as the moment equation incurring second order bias in terms of its nuisance functions. Influence functions in semiparametric/nonparametric models have been known to have this property since the 1980s, going back to the works of Newey, Bickel et al., and Van der Vaart. Therefore, all influence functions are guaranteed to satisfy Neyman Orthogonality. Nevertheless, Neyman Orthogonality can be considerably weaker than double robustness, as although influence functions are guaranteed to have second order bias, double robustness of an influence function requires evaluating the influence function under an appropriate choice of parametrization of the observed data distribution. See for example \citep{tchetgen2010doubly} for an example of an influence function for the semiparametric logistic regression model which is not doubly robust if one of the nuisance parameters is the standard propensity score (i.e., $\text{Pr}(\text{Treatment}=1\mid\text{Covariates})$), however, the same influence function becomes doubly robust if an alternative parametrization replacing the propensity score with the retrospective propensity score which further conditions on the outcome not having occurred (i.e., $\text{Pr}(\text{Treatment}=1\mid\text{Covariates}, \text{Outcome}=0)$), a somewhat surprising result given the central role the standard propensity plays in causal inference. 

Therefore, one could use arguments in \cite{chernozhukov2018double} to establish that the bias of our estimator must be second order, however those arguments are insufficient to establish the exact form of the second order bias, potentially leading to conservative statements about the order the bias.
In fact, note that in \citep{chernozhukov2018double}, all nuisance functions are lumped together as one nuisance function, and the work requires a convergence rate for that nuisance function. On the other hand, in our work, we have two nuisance functions, and we do not require convergence rates on either of them, but we require a convergence rate on the product of the biases. That is, one of the nuisance functions can converge arbitrarily slow (contrary to \citep{chernozhukov2018double}), as long as the other one converges fast enough so that their product is of order smaller than root $n$. The approach of \cite{chernozhukov2018double} generally requires both nuisance functions to converge at rate faster than quarter root $n$. It is the special form of our functional of interest which enabled us to obtain such important property.

Note that besides the aforementioned difference between the two works, due to the fact that we have different arguments for the nuisance functions, throughout the proof, we have to also deal with projection of one function space to the other for our convergence analysis, which is a feature new in our work.

\section{CRITICAL RADIUS}

For the bounding approach in Section \ref{sec:kernel}, we need the following definitions from the statistical learning theory literature \citep{wainwright2019high}.
\begin{definition}[Localized Rademacher Complexity]
For a given $\delta>0$, and function class $\mathcal{F}$,  the localized Rademacher complexity of $\mathcal{F}$ is defined as
\[
\mathcal{R}(\delta,\mathcal{F})\coloneqq
\E_{\epsilon,X}\Bigg[ \sup_{\overset{f\in\mathcal{F}}{\|f\|_{2}\le \delta}} \Big|  \frac{1}{n}\sum_{i=1}^n\epsilon_i f(X_i) \Big| \Bigg],
\]
where $\{X_i\}_{i=1}^n$ are i.i.d. samples from the underlying distribution and $\{\epsilon_i\}_{i=1}^n$ are i.i.d. Rademacher variables taking values in $\{-1,+1\}$ with equal probability, independent of $\{X_i\}_{i=1}^n$.
\end{definition}

\begin{definition}[Critical Radius]
The critical radius of a function class $\mathcal{F}$ is the smallest solution $\delta_n^*$ to the inequality $\mathcal{R}(\delta,\mathcal{F})\le\delta^2$.
\end{definition}

\cite{wainwright2019high} also provided the empirical counterparts of the localized Rademacher complexity and critical radius, which can be used to estimate the corresponding true values.

One of the attractive computational properties of RKHSes is that their localized Rademacher complexity can be determined by their eigendecay. 
For any function class $\mathcal{F}$, let $\mathcal{F}_B\coloneqq\{f\in\mathcal{F}:\|f\|^2_{\mathcal{F}}\le B\}$.
If $\mathcal{F}$ is an RKHS with eigenvalues $\{\mu_j\}_{j=1}^\infty$, then the localized Rademacher complexity of $\mathcal{F}_B$ can be upper bounded as \citep{wainwright2019high}
\[
\mathcal{R}(\delta,\mathcal{F}_B)\le B\sqrt{\frac{2}{n}}\sqrt{\sum_{j=1}^\infty\min\{ \mu_j,\delta^2 \}}.
\]
Therefore, we can find an upper bound for the critical radius of $\mathcal{F}_B$ by finding a solution to the inequality
\[
B\sqrt{\frac{2}{n}}\sqrt{\sum_{j=1}^\infty\min\{ \mu_j,\delta^2 \}}\le\delta^2.
\]

\section{A REMARK ON THE MEASURE OF ILL-POSEDNESS}

\cite{chen2012estimation} also proposed the following measure of ill-posedness in their work 
\[
\tau=\sup_{h\in\sH}\frac{\|h\|_2}{\big\|\textup{P}[h(V_h)|V_q] \big\|_2}.
\]
This measure can also be used for bounding error $\|\hat{h}-h_0\|_2$ simply by
\[
\|\hat{h}-h_0\|_2\le\tau\big\|\textup{P}[\hat{h}(V_h)-h_0(V_h)|V_q] \big\|_2.
\]
However, we note that
\begin{align*}
\tau_h(\delta)&=\delta\sup_{h\in\sH^{|\delta}}\frac{\|h\|_2}{\delta}\\
&=\delta\sup_{h\in\sH^{|\delta}}\frac{\|h\|_2}{\big\|\textup{P}[h(V_h)|V_q] \big\|_2}\cdot\frac{\big\|\textup{P}[h(V_h)|V_q] \big\|_2}{\delta}\\
&\le\delta\sup_{h\in\sH^{|\delta}}\frac{\|h\|_2}{\big\|\textup{P}[h(V_h)|V_q] \big\|_2}\\
&\le\delta\tau.
\end{align*}
Hence, working with $\tau_h(\delta)$, as we do in our approach, leads to tighter bounds.

\section{PROOFS}

\begin{proof}[Proof of Proposition \ref{prop:DRHOIF}]

Due to Assumption \ref{assumption:varind}, the model governing the distribution of the data contains parametric submodels $\mathcal{M}^l=\{P^{l,t}:t\ge0\}$ for $l=1,2,...$ (the parameter of each submodel is denoted by $t$), such that
\begin{itemize}
\item $P^{l,0}=P$
\item $P^{l,t}_X={P}_X$
\item $q^{l,t}(v_q)=q_0(v_q)+tf^l(v_q)$
\item $h^{l,t}(v_h)=h_0(v_h)$	
\end{itemize}
where the sequence $\{f^l\}$ is dense in $\mathcal{L}_2(P_{V_q})$.

Let $S^l(P)$ be the score of the submodel $P^{l,t}$ at $t=0$, that is, 
\[
S^l(P)=\frac{d}{d t}\log P^{l,t}\big|_{t=0}.
\]
We note that
\begin{align*}
\psi(P^{l,t})
&=\E_{P^{l,t}}[
q^{l,t}(V_q)h^{l,t}(V_h)g_1(V)+q^{l,t}(V_q)g_2(V)+h^{l,t}(V_h)g_3(V)+g_4(V)
]\\
&=\E_{P^{l,t}}[
\{q_0(V_q)+tf^l(V_q)\}h_0(V_h)g_1(V)\\
&\quad+\{q_0(V_q)+tf^l(V_q)\}g_2(V)+h_0(V_h)g_3(V)+g_4(V)
].
\end{align*}
Therefore,
\begin{align*}
\frac{d}{d t}\psi(P^{l,t})\big|_{t=0}	
&=\E_P[f^l(V_q)h_0(V_h)g_1(V)+f^l(V_q)g_2(V)]\\
&\quad+\E_P[\{q_0(V_q)h_0(V_h)g_1(V)+q_0(V_q)g_2(V)+h_0(V_h)g_3(V)+g_4(V)\}S^l(P)].
\end{align*}
But we notice that since $q_0(V_q)h_0(V_h)g_1(V)+q_0(V_q)g_2(V)+h_0(V_h)g_3(V)+g_4(V)-\psi_0$ is the influence function of $\psi_0$, we have
\begin{align*}
\frac{d}{d t}\psi(P^{l,t})\big|_{t=0}	
=\E_P[\{q_0(V_q)h_0(V_h)g_1(V)+q_0(V_q)g_2(V)+h_0(V_h)g_3(V)+g_4(V)\}S^l(P)],
\end{align*}
which implies that 
\begin{equation}
\label{eq:flvq}
\E_P[f^l(V_q)h_0(V_h)g_1(V)+f^l(V_q)g_2(V)]=
\E_P[\{h_0(V_h)g_1(V)+g_2(V)\}f^l(V_q)]=0.
\end{equation}
Equation \eqref{eq:flvq} is true for every submodel $l$. Therefore, since the sequence $\{f^l\}$ is dense in $\mathcal{L}_2(P_{V_q})$, we conclude that
$\E[h_0(V_h)g_1(V)+g_2(V)|V_q]=0$.

Similarly, we can show that $\E[q_0(V_q)g_1(V)+g_3(V)|V_h]=0$.

For any choice of $q$, we have
\begin{align*}
&\E[IF(V;\psi_0,q,h_0)]
-\E[IF(V;\psi_0,q_0,h_0)]\\
&\quad=\E[g_1(V)h_0(V_h)\{q(V_q)-q_0(V_q)\}+g_2(V)\{q(V_q)-q_0(V_q)\}]\\
&\quad=\E[\{g_1(V)h_0(V_h)+g_2(V)\}\{q(V_q)-q_0(V_q)\}]\\
&\quad=\E[\E[h_0(V_h)g_1(V)+g_2(V)|V_q]\{q(V_q)-q_0(V_q)\}]=0,
\end{align*}

For any choice of $h$, we have
\begin{align*}
&\E[IF(V;\psi_0,q_0,h)]
-\E[IF(V;\psi_0,q_0,h_0)]\\
&\quad=\E[g_1(V)q_0(V_q)\{h(V_h)-h_0(V_h)\}+g_3(V)\{h(V_h)-h_0(V_h)\}]\\
&\quad=\E[\{g_1(V)q_0(V_q)+g_3(V)\}\{h(V_h)-h_0(V_h)\}]\\
&\quad=\E[\E[q_0(V_q)g_1(V)+g_3(V)|V_h]\{h(V_h)-h_0(V_h)\}]=0,
\end{align*}

\end{proof}

\begin{proof}[Proof of Proposition \ref{prop:separate}]
We note that the function $q$ does not appear in 
\[
\E[ prt(q,h;\dotq,0)]=\E\big[ \dotq(V_q)[h(V_h)g_1(V)+g_2(V)]\big].
\]
Similarly, the function $h$ does not appear in 
\[
\E[ prt(q,h;0,\doth)]=\E\big[ \doth(V_h)[q(V_q)g_1(V)+g_3(V)] \big].
\]
Also, recall that the nuisance functions are assumed to be variation independent.
Therefore, $\hat{q}$ and $\hat{h}$ can be estimated using separate optimizations.

Since $\E\big[ \dotq(V_q)[\hat{h}(V_h)g_1(V)+g_2(V)]\big]=0$ for all $\dotq$, the choice of $\dotq(V_q)=\E[\hat{h}(V_h)g_1(V)+g_2(V)\mid V_q]$ and law of iterative expectations imply that $\E[\hat{h}(V_h)g_1(V)+g_2(V)|V_q]=0$. Similarly, we conclude that $\E[\hat{q}(V_q)g_1(V)+g_3(V)|V_h]=0$.

\end{proof}

\begin{proof}[Proof of Theorem \ref{thm:conds}]
	
We note that
\begin{align*}
&\min_{h}\max_{\dotq }
\E\Big[ \dotq(V_q)[h(V_h)g_1(V)+g_2(V)]-\dotq^2(V_q) \Big]\\
&=\min_{h}-\Big(\min_{\dotq }
\E\Big[ (\dotq(V_q)-\frac{1}{2}[h(V_h)g_1(V)+g_2(V)])^2 \Big]-\frac{1}{4}\E\Big[[h(V_h)g_1(V)+g_2(V)]^2\Big]\Big)
\end{align*}
Similarly,
\begin{align*}
&\min_{q}\max_{\doth }
\E\Big[ \doth(V_h)[q(V_q)g_1(V)+g_3(V)]-\doth^2(V_h) \Big]\\
&=\min_{q}-\Big(\min_{\doth }
\E\Big[ (\doth(V_h)-\frac{1}{2}[q(V_q)g_1(V)+g_3(V)])^2 \Big]-\frac{1}{4}\E\Big[[q(V_q)g_1(V)+g_3(V)]^2\Big]\Big),
\end{align*}
The inner optimizations achieve their optimum at
\[
\dotq^*(V_q)=\frac{1}{2}\E[h(V_h)g_1(V)+g_2(V)|V_q],
\]
\[
\doth^*(V_h)=\frac{1}{2}\E[q(V_q)g_1(V)+g_3(V)|V_h],
\]
Therefore, the outer minimization problems are
\begin{align*}
&\min_{h}
\frac{1}{4}\E\Big[ \E[h(V_h)g_1(V)+g_2(V)|V_q]^2 \Big]\ge0,\\
&\min_{q}
\frac{1}{4}\E\Big[ \E[q(V_q)g_1(V)+g_3(V)|V_h]^2 \Big]\ge0.
\end{align*}
But we note that zero is achievable. This is due to Proposition \ref{prop:separate}, where zero is achieved at the ground-truth nuisance functions $q_0$ and $h_0$. Therefore, the optimizers satisfy
\[
\E[h(V_h)g_1(V)+g_2(V)|V_q]=0,
\]
\[
\E[q(V_q)g_1(V)+g_3(V)|V_h]=0.
\]

\end{proof}

\begin{proof}[Proof of Theorem \ref{thm:convergence}]
	
To simplify the notations, assume the size of each partition of the data is $m$, that is, $n=mL$. Let $(\hat{q}_\ell,\hat{h}_\ell)$ be the estimators of $(q_0,h_0)$ obtained on subsamples $I^c_\ell$. 
For parameters $q$, $h$, and $\psi$ define
\[
\phi(V;q,h)\coloneqq IF(V;\psi,q,h)-\psi,
\]
which implies that $\hat{\psi}_\ell$ can be written as
\[
\hat{\psi}_\ell=\mathbb{P}_m^\ell[\phi(V;\hat{q}_\ell,\hat{h}_\ell)].
\]
Therefore, we have
\begin{align*}
\sqrt{n}\{\hat{\psi}-\psi_0\}=&\sqrt{n}\Big\{\frac{1}{L}\sum_{\ell=1}^L \mathbb{P}_m^\ell[\phi(V;\hat{q}_\ell,\hat{h}_\ell)]-\psi_0]      \Big\}\\
=&\frac{\sqrt{n}}{L\sqrt{m}}\sum_{\ell=1}^L\Bigg\{
\Big\{\mathbb{G}_m^\ell[\phi(V;\hat{q}_\ell,\hat{h}_\ell)]-\mathbb{G}_m^\ell[\phi(V;q_0,h_0)]\Big\}\\
&\quad\quad\quad\quad\quad+\mathbb{G}_m^\ell[\phi(V;q_0,h_0)]\\
&\quad\quad\quad\quad\quad+\sqrt{m}\Big\{ \textup{P}[\phi(V;\hat{q}_\ell,\hat{h}_\ell)]-\psi_0 \Big\}
\Bigg\}
\end{align*}
where for any function $f(V)$,
\[
\mathbb{G}_n[f]=\sqrt{n}\big\{ \mathbb{P}_n[f]-\textup{P}[f]  \big\}
\]
represents the empirical process. 
Recalling that $n=mL$, we have
\begin{align*}
\sqrt{n}\{\hat{\psi}-\psi_0\}
=&\frac{1}{\sqrt{L}}\sum_{\ell=1}^L
\Big\{\mathbb{G}_m^\ell[\phi(V;\hat{q}_\ell,\hat{h}_\ell)]-\mathbb{G}_m^\ell[\phi(V;q_0,h_0)]\Big\}
&&(T_1)
\\
+&\frac{1}{\sqrt{L}}\sum_{\ell=1}^L\mathbb{G}_m^\ell[\phi(V;q_0,h_0)]&&(T_2)\\
+&\frac{1}{\sqrt{L}}\sum_{\ell=1}^L\sqrt{m}\Big\{ \textup{P}[\phi(V;\hat{q}_\ell,\hat{h}_\ell)]- \psi_0 \Big\}&&(T_3)
\end{align*}

We will show that $T_2$ provides us with the term $\sqrt{n}~\hat{\E}[IF(V;\psi_0,q_0,h_0)]$, and under the assumptions of the theorem, $T_1$ and $T_3$ are $o_p(1)$.\\

\noindent
{\large\bf Analysis of $\mathbf{T_1}$}\\

Let $V^{\ell,c}$ be the data in all but the $\ell$-th fold, and define
\[
A_m^\ell\coloneqq
\Big\{\mathbb{G}_m^\ell[\phi(V;\hat{q}_\ell,\hat{h}_\ell)]-\mathbb{G}_m^\ell[\phi(V;q_0,h_0)]\Big\}.
\]
We note that 
\begin{align*}
\textit{var}(A_m^\ell|V^{\ell,c})&=m ~\textit{var}\big( \mathbb{P}_m^\ell[\phi(V;\hat{q}_\ell,\hat{h}_\ell)]-\mathbb{P}_m^\ell[\phi(V;q_0,h_0)] \big|V^{\ell,c} \big)\\
&=\textit{var}\big( \phi(V;\hat{q}_\ell,\hat{h}_\ell)-\phi(V;q_0,h_0) \big|V^{\ell,c} \big)\\
&\le \textup{P}\big[   \{\phi(V;\hat{q}_\ell,\hat{h}_\ell)-\phi(V;q_0,h_0)\}^2    \big|V^{\ell,c}\big]\\
&=\Big\| \phi(V;\hat{q}_\ell,\hat{h}_\ell)-\phi(V;q_0,h_0) \Big\|_2^2\\
&=\Big\|   
g_1(V)\hat{q}_\ell(V_q)\hat{h}_\ell(V_h)
+g_2(V)\hat{q}_\ell(V_q)
+g_3(V)\hat{h}_\ell(V_h)\\
&\quad\quad -g_1(V)q_0(V_q)h_0(V_h)
-g_2(V)q_0(V_q)
-g_3(V)h_0(V_h)
  \Big\|_2^2.
\end{align*}
Note that the last expression is equal to both
\[
\Big\| 
(\hat{h}_\ell(V_h)-h_0(V_h))(g_3(V)+g_1(V)\hat{q}_\ell(V_q))
+(\hat{q}_\ell(V_q)-q_0(V_q))(g_2(V)+g_1(V)h_0(V_h))
  \Big\|_2^2
\]
and
\[
\Big\| 
(\hat{h}_\ell(V_h)-h_0(V_h))(g_3(V)+g_1(V)q_0(V_q))
+(\hat{q}_\ell(V_q)-q_0(V_q))(g_2(V)+g_1(V)\hat{h}_\ell(V_h))
  \Big\|_2^2.
\]
Therefore,
\begin{align*}
\textit{var}(A_m^\ell|V^{\ell,c})
\le 2\min\Bigg\{
&\Big\| 
(\hat{h}_\ell(V_h)-h_0(V_h))(g_3(V)+g_1(V)\hat{q}_\ell(V_q))
\Big\|_2^2\\
&+\Big\| 
(\hat{q}_\ell(V_q)-q_0(V_q))(g_2(V)+g_1(V)h_0(V_h))
\Big\|_2^2
~,\\
&\Big\| 
(\hat{h}_\ell(V_h)-h_0(V_h))(g_3(V)+g_1(V)q_0(V_q))
\Big\|_2^2\\
&+\Big\| 
(\hat{q}_\ell(V_q)-q_0(V_q))(g_2(V)+g_1(V)\hat{h}_\ell(V_h))
\Big\|_2^2
\Bigg\}\\
\le 2\min\Bigg\{
&\sup_v\{(g_3(v)+g_1(v)\hat{q}_\ell(v_q))^2\}
\|\hat{h}_\ell-h_0\|_2^2\\
&+\sup_v\{(g_2(v)+g_1(v)h_0(v_h))^2\}
\|\hat{q}_\ell-q_0\|_2^2
~,\\
&\sup_v\{(g_3(v)+g_1(v)q_0(v_q))^2\}
\|\hat{h}_\ell-h_0\|_2^2\\
&+\sup_v\{(g_2(v)+g_1(v)\hat{h}_\ell(v_h))^2\}
\|\hat{q}_\ell-q_0\|_2^2
\Bigg\}
\end{align*}
By Assumptions \ref{assumption:convergence} and  \ref{assumption:nuisconv}, $\|\hat{q}_\ell-q_0\|_2$ and $\|\hat{h}_\ell-h_0\|_2$ converge to zero in probability, and either the first two, or the second two supremums are finite, which implies that 
$\textit{var}(A_m^\ell|V^{\ell,c})$
converges to zero in probability as $m\rightarrow\infty$.
Also note that $\textup{P}[A_m^\ell|V^{\ell,c}]=0$. Therefore,
\[
\textup{P}[(A_m^\ell)^2|V^{\ell,c}]=\textit{var}(A_m^\ell|V^{\ell,c})\overset{p.}{\rightarrow}0\quad\quad\text{as }m\rightarrow\infty.
\]
Then, by conditional Chebyshev inequality, for all $\delta>0$, we have
\[
P(|A_m^\ell|>\delta\mid V^{\ell,c})\overset{p.}{\rightarrow}0\quad\quad\text{as }m\rightarrow\infty.
\]
That is, $P(|A_m^\ell|>\delta\mid V^{\ell,c})$ is a bounded sequence that converges to zero in probability. Therefore, 
\[
\E\big[P(|A_m^\ell|>\delta\mid V^{\ell,c})\big]=P(|A_m^\ell|>\delta)\overset{p.}{\rightarrow}0\quad\quad\text{as }m\rightarrow\infty.
\]
That is, $A_m^\ell\overset{p.}{\rightarrow}0$ as $m\rightarrow\infty$.

The conclusion holds for all $\ell\in\{1,...,L\}$ and hence, we conclude that $T_1=o_p(1)$.\\

\noindent
{\large\bf Analysis of $\mathbf{T_2}$}
\begin{align*}
T_2
&=\frac{1}{\sqrt{L}}\sum_{\ell=1}^L\mathbb{G}_m^\ell[\phi(V;q_0,h_0)]\\
&=\frac{\sqrt{m}}{\sqrt{L}}\sum_{\ell=1}^L\Big\{\mathbb{P}_m^\ell[\phi(V;q_0,h_0)]-\textup{P}[\phi(V;q_0,h_0)]\Big\}\\
&=\frac{\sqrt{m}}{\sqrt{L}}\frac{1}{m}\sum_{\ell=1}^L\sum_{i=1}^m\Big\{ \phi(V_i;q_0,h_0)-\textup{P}[\phi(V;q_0,h_0)] \Big\}\\
&=\frac{1}{\sqrt{n}}\sum_{i=1}^n\Big\{ \phi(V_i;q_0,h_0)-\textup{P}[\phi(V;q_0,h_0)] \Big\}.
\end{align*}
By central limit theorem, the last expression converges to the distribution $\mathcal{N}(0,\textit{var}~(\phi(V;q_0,h_0)))$ if $\textup{P}[\phi(V;q_0,h_0)^2]<\infty$. Note that $\textup{P}[\phi(V;q_0,h_0)^2]$ can be upper bounded as
\begin{align*}
\textup{P}[\phi(V;q_0&,h_0)^2]\\
&=\|\phi(V;q_0,h_0)\|^2_2\\
&=\| g_1(V)q_0(V_q)h_0(V_h)+g_2(V)q_0(V_q)+g_3(V)h_0(V_h)+g_4(V) \|^2_2\\
&\le2\| \{g_2(V)+g_1(V)h_0(V_h)\}q_0(V_q)\|^2_2+2\|g_3(V)h_0(V_h)\|^2_2+2\|g_4(V) \|^2_2\\
&\le2 \sup_v\{(g_2(v)+g_1(v)h_0(v_h))^2\}\|q_0(V_q)\|^2_2+2\sup_v\{g_3(V)^2\}\|h_0(V_h)\|^2_2+2\sup_v\{g_4(V)^2\},
\end{align*}
or upper bounded as
\begin{align*}
\textup{P}[\phi(V;q_0&,h_0)^2]\\
&=\|\phi(V;q_0,h_0)\|^2_2\\
&=\| g_1(V)q_0(V_q)h_0(V_h)+g_2(V)q_0(V_q)+g_3(V)h_0(V_h)+g_4(V) \|^2_2\\
&\le2\| \{g_3(V)+g_1(V)q_0(V_q)\}h_0(V_h)\|^2_2+2\|g_2(V)q_0(V_q)\|^2_2+2\|g_4(V) \|^2_2\\
&\le2 \sup_v\{(g_3(v)+g_1(v)q_0(v_q))^2\}\|h_0(V_h)\|^2_2+2\sup_v\{g_2(V)^2\}\|q_0(V_q)\|^2_2+2\sup_v\{g_4(V)^2\}.
\end{align*}
Therefore by Assumption \ref{assumption:convergence}, we have $\textup{P}[\phi(V;q_0,h_0)^2]<\infty$.\\

\noindent
{\large\bf Analysis of $\mathbf{T_3}$}\\

Using the double robustness property of the estimating equation, we have

\begin{equation}
\label{eq:T31}
\begin{aligned}
T_3	
&=\frac{1}{\sqrt{L}}\sum_{\ell=1}^L\sqrt{m}\Big\{ \textup{P}[\phi(V;\hat{q}_\ell,\hat{h}_\ell)]-\psi_0  \Big\}\\
&=\frac{1}{\sqrt{L}}\sum_{\ell=1}^L\sqrt{m}\Big\{ \textup{P}[\phi(V;\hat{q}_\ell,\hat{h}_\ell)]-\textup{P}[\phi(V;q_0,\hat{h}_\ell)]  \Big\}\\
&=\sqrt{n}\Big\{ \textup{P}[\phi(V;\hat{q}_\ell,\hat{h}_\ell)]-\textup{P}[\phi(V;q_0,\hat{h}_\ell)]  \Big\}
\end{aligned}
\end{equation}
Therefore (dropping the subscript $\ell$ for readability), we have
\begin{align*}
T_3
&=\sqrt{n}\textup{P}\Big[  
\hat{q}(V_q)\hat{h}(V_h)g_1(V)+\hat{q}(V_q)g_2(V)+\hat{h}(V_h)g_3(V)+g_4(V)\\
&\quad-{q}_0(V_q)\hat{h}(V_h)g_1(V)-{q}_0(V_q)g_2(V)-\hat{h}(V_h)g_3(V)-g_4(V)
   \Big]\\
&=\sqrt{n}\textup{P}\Big[(\hat{h}(V_h)g_1(V)+g_2(V))(\hat{q}(V_q)-{q}_0(V_q))\Big]\\
&=\sqrt{n}\textup{P}\Big[\textup{P}\big[\hat{h}(V_h)g_1(V)+g_2(V)\big|V_q\big](\hat{q}(V_q)-{q}_0(V_q))\Big]
\end{align*}
As seen in the proof of Proposition \ref{prop:DRHOIF}, we have $\E\big[{h}_0(V_h)g_1(V)+g_2(V)\big|V_q\big]=0$. Therefore, for all functions $h$, we have
\[
\E\big[{h}(V_h)g_1(V)+g_2(V)\big|V_q\big]=\E\big[g_1(V)\{{h}_0(V_h)-{h}(V_h)\}\big|V_q\big].
\]
Applying this equality to the last expression for $T_3$ and using Cauchy–Schwarz inequality, we have
\begin{align*}
T_3
&=\sqrt{n}\textup{P}\Big[\textup{P}\big[g_1(V)\{\hat{h}(V_h)-{h}_0(V_h)\}\big|V_q\big](\hat{q}(V_q)-{q}_0(V_q))\Big]\\
&\le\sqrt{n}\Big\|\textup{P}\big[g_1(V)\{\hat{h}(V_h)-{h}_0(V_h)\}\big|V_q\big]\Big\|_2\times\big\|\hat{q}(V_q)-{q}_0(V_q)\big\|_2.
\end{align*}
Note that by Assumption \ref{assumption:convergence},
\begin{align*}
&\Big\|\textup{P}\big[g_1(V)\{\hat{q}(V_q)-{q}_0(V_q)\}\big|V_h\big]\Big\|_2\\
&=\Big\|\textup{P}\big[\textup{P}[g_1(V)|V_q,V_h]\{\hat{q}(V_q)-{q}_0(V_q)\}\big|V_h\big]\Big\|_2\\
&\ge \min_{v_q,v_h}\big|\textup{P}[g_1(V)|V_q=v_q,V_h=v_h]\big| \Big\|\textup{P}\big[\hat{q}(V_q)-{q}_0(V_q)\big|V_h\big]\Big\|_2\\
&\ge \sigma_1 \Big\|\textup{P}\big[\hat{q}(V_q)-{q}_0(V_q)\big|V_h\big]\Big\|_2.
\end{align*}
Therefore, there exists constant $C_1$ such that 
\[
\Big\|\textup{P}\big[\hat{q}(V_q)-{q}_0(V_q)\big|V_h\big]\Big\|_2\le C_1 r_q(n).
\]
This implies that
\[
\big\|\hat{q}(V_q)-{q}_0(V_q)\big\|_2\le C_2 \tau_q\big(r_q(n)\big),
\]
for some constant $C_2$. Therefore,
\begin{align*}
T_3
&\le\sqrt{n}\Big\|\textup{P}\big[g_1(V)\{\hat{h}(V_h)-{h}_0(V_h)\}\big|V_q\big]\Big\|_2\times\big\|\hat{q}(V_q)-{q}_0(V_q)\big\|_2\\
&=O\Big(\sqrt{n}~r_h(n)\tau_q\big(r_q(n)\big)\Big).
\end{align*}
We can also replace $\psi_0$ in equation \eqref{eq:T31} with $\textup{P}[\phi(V;\hat{q}_\ell,{h}_0)]$. Hence, we also have
\begin{align*}
T_3
= O\Big(\sqrt{n}~r_q(n)\tau_h\big(r_h(n)\big)\Big),
\end{align*}
which implies that
\[
T_3
= O\Big(\sqrt{n}~\min\Big\{
r_h(n)\tau_q\big(r_q(n)\big) , 
r_q(n)\tau_h\big(r_h(n)\big)    
\Big\}\Big)=o(1),
\]
where the last equality is due to Assumption \ref{assumption:nuisconv}.

\end{proof}

\begin{proof}[Proof of proposition \ref{prop:proxATE}]

By Theorem \ref{thm:conds}, for arm $A=a$, for function $\tilde{h}_a$, we have
\begin{align*}
&\E[-I(A=a)\tilde{h}_a(W,X)+I(A=a)Y|Z,X]=0\\
&\Leftrightarrow \E[\tilde{h}_a(W,X)-Y|Z,A=a,X]=0\\	
&\Leftrightarrow \E[\tilde{h}(W,a,X)-Y|Z,A=a,X]=0,
\end{align*}
where the function $\tilde{h}$ is defined in the statement of the proposition. 

By Theorem \ref{thm:conds}, for arm $A=a$, for function $\tilde{q}_a$, we have
\begin{align*}
&\E[-I(A=a)\tilde{q}_a(Z,X)+1|W,X]=0\\
&\Leftrightarrow \E[\E[I(A=a)\tilde{q}_a(Z,X)|W,A,X]|W,X]=1 \\	
&\Leftrightarrow \E[I(A=a)|W,X]\E[\tilde{q}_a(Z,X)|W,A=a,X]=1 \\	
&\Leftrightarrow \E[\tilde{q}_a(Z,X)|W,A=a,X]=\frac{1}{P(A=a|W,X)}\\
&\Leftrightarrow \E[\tilde{q}(Z,a,X)|W,A=a,X]=\frac{1}{P(A=a|W,X)},
\end{align*}
where the function $\tilde{q}$ is defined in the statement of the proposition. 

\end{proof}

\begin{proof}[Proof of Proposition \ref{prop:mingen}]
We first show that for any function $h$,
\begin{equation}
\begin{aligned}
\sup_{q\in\mathcal{Q}} \hat{\E}\Big[ q(V_q)[h(V_h)g_1(V)&+g_2(V)]- q^2(V_q) \Big]
-\lambda_{\Q}\|q\|_{\mathcal{Q}}^2\\
&=\frac{1}{4}\{\xi_n(h)\}^\top K_{\Q,n}( \frac{1}{n} K_{\Q,n}+\lambda_\Q I_n)^{-1}\{\xi_n(h)\},
\end{aligned}
\end{equation}
where $\xi_{n}(h)=\frac{1}{n}(h(V_{h_i})g_1(V_i)+g_2(V_i))_{i=1}^n$.

To see this, we note that by the generalized representer theorem \citep{scholkopf2001generalized}, the solution to this maximization will be of the form 
\[
q(v_q)=\sum_{j=1}^n\alpha_jK_\Q(V_{q_j},v_q).
\]
Hence, only the coefficients $\alpha=(\alpha_j)_{j=1}^n$ are needed to be calculated.

Note that we have
\[
\hat{\E}[q(V_q)[h(V_h)g_1(V)+g_2(V)]]=\sum_{i=1}^nq(V_{q_i})\{\xi_n(h)\}_i=\alpha^\top K_{\Q,n}\{\xi_n(h)\},
\]
\[
\hat{\E}[q^2(V_q)]=\frac{1}{n}\alpha^\top K_{\Q,n}^2\alpha,
\]
\[
\|q\|_{\mathcal{Q}}^2=\alpha^\top K_{\Q,n}\alpha.
\]
Therefore, we have
\[
\hat{\E}\Big[ q(V_q)[h(V_h)g_1(V)+g_3(V)]- q^2(V_q) \Big]
-\lambda_{\Q}\|q\|_{\mathcal{Q}}^2
=\alpha^\top K_{\Q,n}\{\xi_n(h)\}-\alpha^\top( \frac{1}{n} K_{\Q,n}^2+\lambda_\Q K_{\Q,n})\alpha.
\]
Taking the derivative and setting it to zero, the optimal coefficients can be obtained as
\[
\alpha^*=\frac{1}{2}( \frac{1}{n} K_{\Q,n}+\lambda_\Q I_n)^{-1}\{\xi_n(h)\},
\]
where $I_n$ is the identity matrix of size $n$.

Consequently we have
\begin{align*}
\sup_{q\in\mathcal{Q}} \hat{\E}\Big[ q(V_q)[h(V_h)g_1(V)&+g_2(V)]
- q^2(V_q) \Big]
-\lambda_{\Q}\|q\|_{\mathcal{Q}}^2\\&=
\frac{1}{4}\{\xi_n(h)\}^\top K_{\Q,n}( \frac{1}{n} K_{\Q,n}+\lambda_\Q I_n)^{-1}\{\xi_n(h)\}.
\end{align*}

Therefore, the outer minimization problem in \eqref{eq:optkernelh} is reduced to the following:
\[
\hat{h}
=\arg\min_{h\in\mathcal{H}} \{\xi_n(h)\}^\top \Gamma\{\xi_n(h)\}+\lambda_{\sH}\|h\|_{\mathcal{H}}^2,
\]
where $\Gamma=\frac{1}{4}K_{\Q,n}( \frac{1}{n} K_{\Q,n}+\lambda_\Q I_n)^{-1}$.

We note that by the generalized representer theorem \citep{scholkopf2001generalized}, the solution to this minimization problem will be of the form 
\[
h(v_h)=\sum_{j=1}^n\alpha_jK_\sH(V_{h_j},v_h).
\]
Hence, only the coefficients $\alpha=(\alpha_j)_{j=1}^n$ are needed to be calculated.

Recall that $\xi_{n}(h)=\frac{1}{n}(h(V_{h_i})g_1(V_i)+g_2(V_i))_{i=1}^n$; that is,
\[
\xi_n(h)=\frac{1}{n}(diag(g_{1,n})K_{\sH,n}\alpha+g_{2,n}),
\] 
where $diag(g_{1,n})$ is a diagonal matrix with $g_1(V_i)$ as the $i$-th diagonal entry, and $g_{2,n}\coloneqq[g_2(V_i)\cdots g_2(V_i)]^\top$.
Also note that
\[
\|h\|_{\mathcal{H}}^2=\alpha^\top K_{\sH,n}\alpha.
\]
Therefore, we have
\begin{align*}
&\min_{h\in\mathcal{H}} \{\xi_n(h)\}^\top \Gamma\{\xi_n(h)\}+\lambda_{\sH}\|h\|_{\mathcal{H}}^2\\
&=\min_{\alpha\in\mathbb{R}^n} 
\frac{1}{n^2}\alpha^\top K_{\sH,n}diag(g_{1,n}) \Gamma diag(g_{1,n}) K_{\sH,n}\alpha 
+\frac{2}{n^2}g_{2,n}^\top  \Gamma diag(g_{1,n}) K_{\sH,n}\alpha 
+\lambda_\sH\alpha^\top K_{\sH,n}\alpha+c,
\end{align*}
which is solved by
\[
\alpha^*=-\Big( K_{\sH,n}diag(g_{1,n}) \Gamma diag(g_{1,n}) K_{\sH,n}+n^2\lambda_\sH K_{\sH,n}\Big)^\dagger K_{\sH,n}diag(g_{1,n}) \Gamma g_{2,n}.
\]

\end{proof}

\begin{proof}[Proof of Theorem \ref{thm:nuisanceconv}]

The proof is for the most part the same as the proof of Theorem 1 in \citep{dikkala2020minimax}. 
For a function $h$, define its empirical $L_2$ norm as $\|h\|_{2,n}\coloneqq\sqrt{\mathbb{P}_n[h^2]}$. 
Define
\begin{align*}
&\Psi(h,q)\coloneqq\E\big[\{g_1(V)h(V_h)+g_2(V)\}q(V_q)  \big]= \E\big[g_1(V)\{h(V_h)-h_0(V_h)\}q(V_q)  \big],\\
&\Psi_n(h,q)\coloneqq \hat{\E}\big[\{g_1(V)h(V_h)+g_2(V)\}q(V_q)  \big],
\end{align*}
and 
\begin{align*}
&\Psi^{\lambda}(h,q)\coloneqq \Psi(h,q)-\frac{1}{2}\|q\|^2_{2}-\frac{\lambda_{\Q}}{2}\|q\|_{\mathcal{Q}}^2,\\
&\Psi^{\lambda}_n(h,q)\coloneqq\Psi_n(h,q)-\|q\|^2_{2,n}-\lambda_{\Q}\|q\|_{\mathcal{Q}}^2.
\end{align*}

Note that the minimax optimization \eqref{eq:optkernelh} can be written as
\begin{align*}
\hat{h}
&=\arg\min_{h\in\mathcal{H}}\sup_{q\in\mathcal{Q}} \hat{\E}\Big[ q(V_q)[h(V_h)g_1(V)+g_2(V)]- q^2(V_q) \Big]
-\lambda_{\Q}\|q\|_{\mathcal{Q}}^2+\lambda_{\sH}\|h\|_{\mathcal{H}}^2\\
&=\arg\min_{h\in\mathcal{H}}\sup_{q\in\mathcal{Q}} 
\Psi^{\lambda}_n(h,q)+\lambda_{\sH}\|h\|_{\mathcal{H}}^2.
\end{align*}
The proof proceeds in 3 steps:\\

{\bf Step 1.} We will show that with probability $1-2\zeta$,
\[
\sup_{q\in\Q}\Big\{\Psi_n(\hat{h},q)-\Psi_n(h_0,q)-\|q\|^2_{2,n}-\lambda_{\Q}\|q\|_{\mathcal{Q}}^2\Big\}
\le\lambda_\sH\big( \|h_0\|^2_\sH-\|\hat{h}\|^2_\sH  \big)+O(\delta^2).
\]

{\bf Step 2.} For any function $h$, define $q_h\coloneqq\textup{P}\big[g_1(V)\{h(V_h)-h_0(V_h)\}\big|V_q\big]$. We will show that if $\|q_{\hat{h}}\|_2\ge\delta$, then with probability $1-\zeta$,
\[
\sup_{q\in\Q}\Big\{\Psi_n(\hat{h},q)-\Psi_n(h_0,q)-\|q\|^2_{2,n}-\lambda_{\Q}\|q\|_{\mathcal{Q}}^2\Big\}
\ge
\frac{\delta}{2}R(\hat{h})-C\delta^2-\lambda_\sH\big( \|h_0\|^2_\sH+\|\hat{h}\|^2_\sH  \big).
\]

{\bf Step 3.} Finally, using Steps 1 and 2, if $\|q_{\hat{h}}\|_2\ge\delta$, with probability $1-3\zeta$, we have
\[
\frac{\delta}{2}R(\hat{h})\le 
2\lambda_\sH\|h_0\|^2_\sH+O(\delta^2),
\]
which concludes that
\begin{equation}
\label{eq:step3}
R(\hat{h})=
O\Big(\delta+\frac{\lambda_\sH}{\delta}\|h_0\|^2_\sH\Big).
\end{equation}
Note that $\|q_{\hat{h}}\|_2=R(\hat{h})$. Hence, $\|q_{\hat{h}}\|_2\le\delta$ implies $R(\hat{h})\le\delta$. Therefore, \eqref{eq:step3} in either case holds.

~\\\\
We next prove the inequalities in Steps 1 and 2. The main ingredients in the proofs are the following two results
\begin{theorem}[\cite{wainwright2019high}, Theorem 14.1]
\label{thm:wain}
Let $\mathcal{F}$ be a star-convex, 1-uniformly bounded function class with critical radius $\delta_n$. Then for any $t\ge\delta_n$, we have
\[
\Big|\|f\|_{2,n}^2-\|f\|_{2}^2\Big|\le\frac{1}{2}\|f\|_{2}^2+\frac{t^2}{2},\quad\text{for all }f\in\mathcal{F},
\]
with probability at least $1-c_1e^{-c_2nt^2}$.
\end{theorem}
Theorem \ref{thm:wain} implies that for the choice of $t$ equal to $\delta=\delta_n+c_0\sqrt{\frac{\log(c_1/\zeta)}{n}}$, where $\delta_n$ is an upper bound on the critical radius of $Q_U$, for some constants $c_0$ and $c_1$, we have
\[
\Big|\|q\|_{2,n}^2-\|q\|_{2}^2\Big|\le\frac{1}{2}\|q\|_{2}^2+\frac{1}{2}\delta^2,\quad\text{for all }q\in\Q_{U},
\]
with probability at least $1-c_1e^{-c_2n\delta^2}$, and hence, with probability larger than $1-\zeta$.

Moreover, for any function $q\in\Q$ such that $\|q\|^2_\Q\ge U$, since $\Q$ is assumed to be star-convex, the function $\frac{q\sqrt{U}}{\|q\|_\Q}$ is in $\Q_{U}$. Therefore, again by Theorem \ref{thm:wain}, we have
\[
\Big|\|q\|_{2,n}^2-\|q\|_{2}^2\Big|\le\frac{1}{2}\|q\|_{2}^2+\frac{1}{2}\delta^2\frac{\|q\|^2_\Q}{U},
\]
with probability larger than $1-\zeta$. The last two displays imply that with probability larger than $1-\zeta$,
\begin{equation}
\label{eq:ing1}
\Big|\|q\|_{2,n}^2-\|q\|_{2}^2\Big|\le\frac{1}{2}\|q\|_{2}^2+\frac{\delta^2}{2U}\|q\|^2_\Q+\frac{1}{2}\delta^2,\quad\text{for all }q\in\Q.
\end{equation}
Therefore, with probability at least $1-\zeta$,
\[
\lambda_\Q\|q\|_\Q^2+\|q\|_{2,n}^2
\ge (\lambda_\Q-\frac{\delta^2}{2U})\|q\|_\Q^2
+\frac{1}{2}\|q\|_{2}^2-\frac{\delta^2}{2}.
\]
Therefore, assuming that $\lambda_\Q\ge\frac{\delta^2}{U}$,
\begin{equation}
\label{eq:ing2}
\lambda_\Q\|q\|_\Q^2+\|q\|_{2,n}^2
\ge \frac{\lambda_\Q}{2}\|q\|_\Q^2
+\frac{1}{2}\|q\|_{2}^2-\frac{\delta^2}{2}.
\end{equation}

The second key result used in the proof is an extension of Theorem \ref{thm:wain}, proposed in \citep{foster2019orthogonal}. Since $g_1$ and $g_2$ are bounded and $\sH$ is a class of bounded functions, the function $\{g_1(V)h(V_h)+g_2(V)\}q(V_q)$ is Lipschitz with respect to $q$. Therefore, by Lemma 11 in \citep{foster2019orthogonal}, with probability $1-\zeta$,
\[
\big|\Psi_n(h_0,q)-\Psi(h_0,q)\big|\le O(\delta\|q\|_{2}+\delta^2),
\quad\text{for all }q\in\Q_{U}.
\]
Moreover, for any function $q\in\Q$ such that $\|q\|^2_\Q\ge U$, since $\Q$ is assumed to be star-convex, the function $\frac{q\sqrt{U}}{\|q\|_\Q}$ is in $\Q_{U}$. Therefore,
\[
\big|\Psi_n(h_0,q)-\Psi(h_0,q)\big|\le O(\delta\|q\|_{2}+\frac{\delta^2}{\sqrt{U}}\|q\|_\Q),
\]
with probability larger than $1-\zeta$. The last two displays imply that with probability larger than $1-\zeta$,
\begin{equation}
\label{eq:ing3}
\big|\Psi_n(h_0,q)-\Psi(h_0,q)\big|\le O(\delta\|q\|_{2}+\delta^2+\frac{\delta^2}{\sqrt{U}}\|q\|_\Q),
\quad\text{for all }q\in\Q.
\end{equation}

For any function $h$, define $q_h\coloneqq\textup{P}\big[g_1(V)\{h(V_h)-h_0(V_h)\}\big|V_q\big]$, which is a function in $\Q$.
Again, by Lemma 11 in \citep{foster2019orthogonal}, 
since $\delta_n$ is an upper bound on the critical radius of $\mathcal{G}_B$, for all functions $h$ such that $h-h_0\in\sH_B$, 
with probability $1-\zeta$,
\begin{align*}
\big|
\big(\Psi_n(h,q_h)-\Psi_n(h_0,q_h)\big)-
\big(\Psi(h,q_h)-\Psi(h_0,q_h)\big)
\big|
&\le O(\delta\|g_1(V)\{h(V_h)-h_0(V_h)\}q_h(V_q)\|_{2}+\delta^2)\\
&=O(\delta\|q_h(V_q)\|_{2}+\delta^2).
\end{align*}
Moreover, for any function $h\in\sH$ such that $\|h-h_0\|^2_\sH\ge B$, since $\sH$ is assumed to be star-convex, the function $\frac{(h-h_0)\sqrt{B}}{\|h-h_0\|_\sH}$ is in $\sH_{B}$. Therefore,
\[
\big|
\big(\Psi_n(h,q_h)-\Psi_n(h_0,q_h)\big)-
\big(\Psi(h,q_h)-\Psi(h_0,q_h)\big)
\big|
\le O(\delta\|q_h(V_q)\|_{2}\frac{\|h-h_0\|^2_\sH}{B}+\delta^2\frac{\|h-h_0\|^2_\sH}{B}),
\]
with probability larger than $1-\zeta$. The last two displays imply that with probability larger than $1-\zeta$, 
\begin{equation}
\label{eq:ing4}
\begin{aligned}
&\big|
\big(\Psi_n(h,q_h)-\Psi_n(h_0,q_h)\big)-
\big(\Psi(h,q_h)-\Psi(h_0,q_h)\big)
\big|\\
&\quad\le 
O(\delta\|q_h(V_q)\|_{2}+\delta^2+\delta\|q_h(V_q)\|_{2}\frac{\|h-h_0\|^2_\sH}{B}+\delta^2\frac{\|h-h_0\|^2_\sH}{B}),
\quad\text{for all }h\in\sH.
\end{aligned}
\end{equation}

Equipped with inequalities \eqref{eq:ing1}, \eqref{eq:ing2}, \eqref{eq:ing3}, and \eqref{eq:ing4}, we now present the proofs of Steps 1 and 2.\\\\\\

\emph{Proof of Step 1.}

We note that,
\begin{align*}
\sup_{q\in\Q}\Psi^{\lambda}_n(\hat{h},q)
&=\sup_{q\in\Q}\Big\{\Psi_n(\hat{h},q)-\|q\|^2_{2,n}-\lambda_{\Q}\|q\|_{\mathcal{Q}}^2\Big\}\\
&=\sup_{q\in\Q}\Big\{\Psi_n(\hat{h},q)-\Psi_n(h_0,q)+\Psi_n(h_0,q)-\|q\|^2_{2,n}-\lambda_{\Q}\|q\|_{\mathcal{Q}}^2\Big\}\\
&\ge\sup_{q\in\Q}\Big\{\Psi_n(\hat{h},q)-\Psi_n(h_0,q)-\|q\|^2_{2,n}-\lambda_{\Q}\|q\|_{\mathcal{Q}}^2\Big\}
-\sup_{q\in\Q}\Psi^{\lambda}_n(h_0,q).
\end{align*}
That is,
\[
\sup_{q\in\Q}\Big\{\Psi_n(\hat{h},q)-\Psi_n(h_0,q)-\|q\|^2_{2,n}-\lambda_{\Q}\|q\|_{\mathcal{Q}}^2\Big\}
\le\sup_{q\in\Q}\Psi^{\lambda}_n(h_0,q)+ \sup_{q\in\Q}\Psi^{\lambda}_n(\hat{h},q).
\]
Moreover, by definition,
\[
\sup_{q\in\Q}\Psi^{\lambda}_n(\hat{h},q)+\lambda_\sH\|\hat{h}\|^2_\sH 
\le\sup_{q\in\Q} \Psi^{\lambda}_n(h_0,q)+\lambda_\sH\|h_0\|^2_\sH.
\]
Hence,
\[
\sup_{q\in\Q}\Psi^{\lambda}_n(\hat{h},q)
\le\sup_{q\in\Q} \Psi^{\lambda}_n(h_0,q)+\lambda_\sH(\|h_0\|^2_\sH-\|\hat{h}\|^2_\sH ).
\]
Therefore, we have
\begin{equation}
\label{eq:step1_1}
\sup_{q\in\Q}\Big\{\Psi_n(\hat{h},q)-\Psi_n(h_0,q)-\|q\|^2_{2,n}-\lambda_{\Q}\|q\|_{\mathcal{Q}}^2\Big\}
\le2\sup_{q\in\Q}\Psi^{\lambda}_n(h_0,q)+\lambda_\sH(\|h_0\|^2_\sH-\|\hat{h}\|^2_\sH ).
\end{equation}

By inequalities \eqref{eq:ing2} and \eqref{eq:ing3}, with probability $1-2\zeta$, for some constant $c$ we have
\begin{equation}
\label{eq:step1_0}
\begin{aligned}
\sup_{q\in\Q}\Psi^{\lambda}_n(h_0,q)
&=\sup_{q\in\Q}\Big\{\Psi_n(h_0,q)-\|q\|^2_{2,n}-\lambda_{\Q}\|q\|_{\mathcal{Q}}^2\Big\}\\
&\le\sup_{q\in\Q}\Big\{
\Psi(h_0,q)+ c(\delta\|q\|_{2}+\delta^2+\frac{\delta^2}{\sqrt{U}}\|q\|_\Q)
-\frac{\lambda_\Q}{2}\|q\|_\Q^2
-\frac{1}{2}\|q\|_{2}^2+\frac{\delta^2}{2}\Big\}\\
&=\sup_{q\in\Q}\Big\{
 c(\delta\|q\|_{2}+\delta^2+\frac{\delta^2}{\sqrt{U}}\|q\|_\Q)
-\frac{\lambda_\Q}{2}\|q\|_\Q^2
-\frac{1}{2}\|q\|_{2}^2+\frac{\delta^2}{2}\Big\},
\end{aligned}
\end{equation}
where the last equality is due to the fact that $\Psi(h_0,q)=0$.
We note that since we assumed that $\lambda_\Q\ge\frac{\delta^2}{U}$, we have
\begin{equation}
\label{eq:step1_2}
\begin{aligned}
&\sup_{q\in\Q}\Big\{
c\delta\|q\|_{2}-\frac{1}{2}\|q\|_{2}^2
+c\frac{\delta^2}{\sqrt{U}}\|q\|_\Q-\frac{\lambda_\Q}{2}\|q\|_\Q^2\Big\}\\
&\le\sup_{q\in\Q}\Big\{ c\delta\|q\|_{2}-\frac{1}{2}\|q\|_{2}^2  \Big\}
+\sup_{q\in\Q}\Big\{  c\frac{\delta^2}{\sqrt{U}}\|q\|_\Q-\frac{\delta^2}{2U}\|q\|_\Q^2 \Big\}\\
&\le \frac{c^2\delta^2}{2}+\frac{c^2\delta^2}{2}=O(\delta^2).
\end{aligned}
\end{equation}

From inequalities \eqref{eq:step1_0}, \eqref{eq:step1_1}, and \eqref{eq:step1_2} we conclude that 
\[
\sup_{q\in\Q}\Big\{\Psi_n(\hat{h},q)-\Psi_n(h_0,q)-\|q\|^2_{2,n}-\lambda_{\Q}\|q\|_{\mathcal{Q}}^2\Big\}
\le\lambda_\sH\big( \|h_0\|^2_\sH-\|\hat{h}\|^2_\sH  \big)+O(\delta^2).
\]
~\\

\emph{Proof of Step 2.}

Recall that for any function $h$,  $q_h\coloneqq\textup{P}\big[g_1(V)\{h(V_h)-h_0(V_h)\}\big|V_q\big] $, which is a function in $\Q$.
Suppose $\|q_{\hat{h}}\|_2\ge\delta$ and $\alpha=\frac{\delta}{2\|q_{\hat{h}}\|_2}\in[0,\frac{1}{2}]$. Since $\Q$ is star-convex and $q_{\hat{h}}\in\Q$, then also $\alpha q_{\hat{h}}\in\Q$. Therefore,
\begin{equation}
\label{eq:step2_0}
\begin{aligned}
&\sup_{q\in\Q}\Big\{\Psi_n(\hat{h},q)-\Psi_n(h_0,q)-\|q\|^2_{2,n}-\lambda_{\Q}\|q\|_{\mathcal{Q}}^2\Big\}\\
&\quad\quad\ge
\alpha\Big(\Psi_n(\hat{h},q_{\hat{h}})-\Psi_n(h_0,q_{\hat{h}})\Big)
-\alpha^2\Big(\|q_{\hat{h}}\|^2_{2,n}+\lambda_{\Q}\|q_{\hat{h}}\|_{\mathcal{Q}}^2\Big).
\end{aligned}
\end{equation}

By inequality \eqref{eq:ing4}, with probability $1-\zeta$, we have
\begin{align*}
&
\alpha\Big(\Psi_n(\hat{h},q_{\hat{h}})-\Psi_n(h_0,q_{\hat{h}})\Big)
\\
&\quad\ge 
\alpha\Big(\Psi(\hat{h},q_{\hat{h}})-\Psi(h_0,q_{\hat{h}})\Big)
-\alpha O(\delta\|q_{\hat{h}}(V_q)\|_{2}+\delta^2+\delta\|q_{\hat{h}}(V_q)\|_{2}\frac{\|\hat{h}-h_0\|^2_\sH}{B}+\delta^2\frac{\|\hat{h}-h_0\|^2_\sH}{B}).
\end{align*}
Hence, recalling that $\alpha=\frac{\delta}{2\|q_{\hat{h}}\|_2}\in[0,\frac{1}{2}]$, and since $\Psi(h_0,q_{\hat{h}})=0$, there exists constant $c$, such that
\begin{align*}
\alpha\Big(\Psi_n(\hat{h},q_{\hat{h}})-\Psi_n(h_0,q_{\hat{h}})\Big)
\ge 
\alpha\Psi(\hat{h},q_{\hat{h}})
-c\delta^2-c\delta^2\frac{\|\hat{h}-h_0\|^2_\sH}{B}.
\end{align*}
Also, we note that
\begin{align*}
\alpha\Psi(\hat{h},q_{\hat{h}})
&=\frac{\delta}{2\|q_{\hat{h}}\|_2}
\text{P}\big[g_1(V)\{\hat{h}(V_h)-h_0(V_h)\}q_{\hat{h}}(V_q)  \big]\\
&=\frac{\delta}{2\|q_{\hat{h}}\|_2}
\text{P}\big[\text{P}[g_1(V)\{\hat{h}(V_h)-h_0(V_h)\}|V_q]^2  \big]\\
&=\frac{\delta\|q_{\hat{h}}\|_2^2}{2\|q_{\hat{h}}\|_2}
=\frac{\delta}{2}R(\hat{h}).
\end{align*}
Therefore,
\begin{equation}
\label{eq:step2_1}
\begin{aligned}
\alpha\Big(\Psi_n(\hat{h},q_{\hat{h}})-\Psi_n(h_0,q_{\hat{h}})\Big)
\ge
\frac{\delta}{2}R(\hat{h})-c\delta^2-c\delta^2\frac{\|\hat{h}-h_0\|^2_\sH}{B}.
\end{aligned}
\end{equation}

By inequality \eqref{eq:ing1}, the fact that $\lambda_\Q\ge\frac{\delta^2}{U}$, and the fact that $\|q_{\hat{h}}\|_{\mathcal{Q}}^2\le L^2\|\hat{h}-h_0\|^2_\sH$, we have
\begin{equation}
\label{eq:step2_2}
\begin{aligned}
-\alpha^2\Big(\|q_{\hat{h}}\|^2_{2,n}+\lambda_{\Q}\|q_{\hat{h}}\|_{\mathcal{Q}}^2\Big)
&\ge
-\alpha^2\Big(
\frac{3}{2}\|q_{\hat{h}}\|^2_{2}+\frac{\delta^2}{2U}\|q_{\hat{h}}\|_{\mathcal{Q}}^2+\frac{1}{2}\delta^2
+\lambda_{\Q}\|q_{\hat{h}}\|_{\mathcal{Q}}^2\Big)\\
&\ge
-\alpha^2
\frac{3}{2}\|q_{\hat{h}}\|^2_{2}
-\frac{\delta^2}{2U}\|q_{\hat{h}}\|_{\mathcal{Q}}^2
-\frac{1}{2}\delta^2
-\lambda_{\Q}\|q_{\hat{h}}\|_{\mathcal{Q}}^2\\
&\ge
-\delta^2
-(\frac{\delta^2}{2U}+\lambda_{\Q})\|q_{\hat{h}}\|_{\mathcal{Q}}^2\\
&\ge
-\delta^2
-\frac{3\lambda_{\Q}}{2}\|q_{\hat{h}}\|_{\mathcal{Q}}^2\\
&\ge
-\delta^2
-\frac{3\lambda_{\Q}}{2}L^2\|\hat{h}-h_0\|^2_\sH.
\end{aligned}
\end{equation}

From inequalities \eqref{eq:step2_0}, \eqref{eq:step2_1}, and \eqref{eq:step2_2} we have
\begin{align*}
&\sup_{q\in\Q}\Big\{\Psi_n(\hat{h},q)-\Psi_n(h_0,q)-\|q\|^2_{2,n}-\lambda_{\Q}\|q\|_{\mathcal{Q}}^2\Big\}\\
&\quad\quad\ge
\frac{\delta}{2}R(\hat{h})-c\delta^2-c\delta^2\frac{\|\hat{h}-h_0\|^2_\sH}{B}
-\delta^2
-\frac{3\lambda_{\Q}}{2}L^2\|\hat{h}-h_0\|^2_\sH\\
&\quad\quad\ge
\frac{\delta}{2}R(\hat{h})-c\delta^2-cU\lambda_\Q\frac{\|\hat{h}-h_0\|^2_\sH}{B}
-\delta^2
-\frac{3\lambda_{\Q}}{2}L^2\|\hat{h}-h_0\|^2_\sH\\
&\quad\quad=
\frac{\delta}{2}R(\hat{h})-c\delta^2-\delta^2
-\lambda_{\Q}\Big(\frac{cU}{B}
+\frac{3}{2}L^2\Big)
\|\hat{h}-h_0\|^2_\sH\\
&\quad\quad\ge
\frac{\delta}{2}R(\hat{h})-c\delta^2-\delta^2
-\lambda_{\Q}\Big(\frac{2cU}{B}
+3L^2\Big)
\big( \|h_0\|^2_\sH+\|\hat{h}\|^2_\sH  \big)\\
&\quad\quad=
\frac{\delta}{2}R(\hat{h})-c\delta^2-\delta^2
-\lambda_{\Q}\Big(4cL^2
+3L^2\Big)
\big( \|h_0\|^2_\sH+\|\hat{h}\|^2_\sH  \big).
\end{align*}
Assuming $\lambda_\Q\Big(4cL^2
+3L^2\Big)\le\lambda_\sH$, for some constant $C$ we conclude that if $\|q_{\hat{h}}\|_2\ge\delta$, then with probability $1-\zeta$,
\[
\sup_{q\in\Q}\Big\{\Psi_n(\hat{h},q)-\Psi_n(h_0,q)-\|q\|^2_{2,n}-\lambda_{\Q}\|q\|_{\mathcal{Q}}^2\Big\}
\ge
\frac{\delta}{2}R(\hat{h})-C\delta^2-\lambda_\sH\big( \|h_0\|^2_\sH+\|\hat{h}\|^2_\sH  \big).
\]

\end{proof}

\section{ADDITIONAL MATERIALS FOR SECTION \ref{sec:exp}}
\subsection{Data Generating Process in the Synthetic-Data Experiments}
We first describe the data generating mechanism of the scenario considered in the main text, which is the same as in \citep{cui2020semiparametric}.
Covariates $X$ are generated from a multivariate normal distribution $N(\Gamma_x,\Sigma_x)$. We then generate $A$ conditional on $X$ from a Bernoulli distribution.

Next, we generate $Z,W,U$ from the following multivariate normal distribution,
\[
\left( Z,W,U\right) |A,X\sim MVN\left( \left(
\begin{array}{c}
\alpha _{0}+\alpha _{a}A+\alpha _{x}X \\
\mu _{0}+\mu _{a}A+\mu _{x}X \\
\kappa _{0}+\kappa _{a}A+\kappa _{x}X
\end{array}
\right) ,\Sigma=\left(
\begin{array}{ccc}
\sigma _{z}^{2} & \sigma _{zw} & \sigma _{zu} \\
\sigma _{zw} & \sigma _{w}^{2} & \sigma _{wu} \\
\sigma _{zu} & \sigma _{wu} & \sigma _{u}^{2}
\end{array}
\right) \right).
\]

Finally, $Y$ is generated from
\begin{eqnarray*}
\E\left( Y|W,U,A,Z,X\right)  &=&\E\left( Y|U,A,Z,X\right) +\omega \left\{
W-\E\left( W|U,A,Z,X\right) \right\}  \\
&=&\E\left( Y|U,A,X\right) +\omega \left\{ W-\E\left( W|U,X\right) \right\}  \\
&=&b_{0}+b_{a}A+b_{x}X+b_{w}\E\left( W|U,X\right) +\omega \left\{ W-\E\left( W|U,X\right)
\right\}  \\
&=&b_{0}+b_{a}A+b_{x}X+\left( b_{w}-\omega \right) \E\left( W|U,X\right) +\omega W,
\end{eqnarray*}
where
\[
\E\left( W|U,X\right) =\E\left( W|U,A,Z,X\right) =\mu _{0}+\mu _{x}X+\frac{\sigma _{wu}}{
\sigma _{u}^{2}}\left( U-\kappa _{0}-\kappa_{x}X\right).
\]
The parameters are set as follows:
\begin{itemize}
\setlength\itemsep{1em}

\item $\Gamma_x=(0.25,0.25)^T$, $\Sigma_x=\left(
\begin{array}{ccc}
\sigma_x^2 & 0\\
0 & \sigma_x^2\\
\end{array}
\right)$, $\sigma_x=0.25$.

\item $\Pr \left(A=1|X\right)=\left[1+ \exp\{(0.125,0.125)^TX\}\right]^{-1}$.

\item $\alpha_0= 0.25$, $\alpha_a= 0.125$, $\alpha_x= (0.25,0.25)^T$.

\item $\mu_0= 0.25$, $\mu_a= 0.25$, $\mu_x= (0.25,0.25)^T$.

\item $\kappa_0= 0.25$, $\kappa_a= 0.25$, $\kappa_x= (0.25,0.25)^T$.

\item $\Sigma=\left(
\begin{array}{ccc}
1 & 0.25 & 0.5 \\
0.25 & 1 & 0.5 \\
0.5 & 0.5 & 1
\end{array}
\right).$

\item $b_0= 2$, $b_a= 2$, $b_x= (0.25,0.25)^T$, $b_w=4$, $\omega=2$.
\end{itemize}
The validity of relevant assumptions imposed in Section \ref{sec:apps} was proved in \citep{cui2020semiparametric}.
This scenario sets dimensions of $X$, $Z$, $W$ to 2, 1, 1. Hyper parameters $\lambda_Q^q$, $\lambda_Q^h$, $\lambda_H^q$, and $\lambda_H^h$, as well as the kernel bandwidth are chosen by cross validation with minimizing the projected risk. 

\subsection{Multidimensional Proxies}
We further consider a more complicated scenario where the dimensions of $X$, $Z$, $W$ increase to 5, 2, 2. The data are generated by
\begin{equation}
    U \sim \text{Normal}(0, \sigma_U^2)
\end{equation}
\begin{equation}
    X|U \sim \text{Normal}(\mu_{UX}U, \sigma_X^2)
\end{equation}
\begin{equation}
    A|X, U \sim \text{Bernoulli}(1/(1 + \exp(t_A + t_X X + t_Z (M_{AZ} + \mu_{XZ} X + \mu_{UZ} U) + 0.5 t_Z \sigma_Z^2 t_Z^\top)))
\end{equation}
\begin{equation}
    Z|A, X, U \sim \text{Normal}(\mu_{AZ} A + \mu_{XZ} X + \mu_{UZ} U, \sigma_Z^2)
\end{equation}
\begin{equation}
    W|X, U \sim \text{Normal}(\mu_{XW} X + \mu_{UW} U, \sigma_W^2)
\end{equation}
\begin{equation}
    Y|A, X, U \sim \text{Normal}(\mu_{AY} A + \mu_{WY}W + \mu_{XY} X + \mu_{UY} U, \sigma_Y^2)
\end{equation}
The parameters are set as follows
\begin{itemize}
    \item $\mu_{U} = 0$, $\sigma_U = 0.3$.
    \item $\mu_{UX} = \left(\begin{array}{c}
0.4 \\
0.525 \\
0.650\\
0.775\\
0.9
\end{array}\right)$,  $\sigma_X^2 = \left(\begin{array}{ccccc}
0.3 & 0 & 0 & 0 & 0 \\
0 & 0.3  & 0 & 0 & 0 \\
0 & 0 & 0.3  & 0 & 0 \\
0 & 0 & 0 & 0.3  & 0 \\
0 & 0 & 0 & 0 & 0.3 
\end{array}
    \right).$
    \item $t_Z = (0.9, 0.4)$, $t_X = (0.9, 0.775, 0.65, 0.525, 0.4)$, $t_A = -t_Z \mu_{AZ} - t_Z 
    \sigma_Z^2 t_Z^\top = -0.981$.
    \item $\mu_{AZ} = (0.5, 0.6)$, $\mu_{XZ} = \left(\begin{array}{ccccc}
0.4 & 0.511 & 0.622 & 0.733 & 0.844 \\
0.456 & 0.567  & 0.678 & 0.789 & 0.9 
\end{array}
    \right)$, $\mu_{UZ} = (0.8, 0.9)$,  $\sigma_Z^2 = \left(\begin{array}{cc}
0.3 & 0 \\
0 & 0.3 
\end{array}
    \right).$
    \item $\mu_{XW} = \left(\begin{array}{ccccc}
0.9 & 0.789 & 0.678 & 0.567  & 0.455 \\
0.844 & 0.733 & 0.622 & 0.511 & 0.4
\end{array}
    \right)$, $\mu_{UW} = (0.8, 0.9)$, $\sigma_W^2 = \left(\begin{array}{cc}
0.3 & 0 \\
0 & 0.3 
\end{array}
    \right).$
    \item $\mu_{AY} = 2$, $\mu_{WY} = (0.4, 0.9)$, $\mu_{XY} = (0.4, 0.525, 0.65, 0.775, 0.9)$, $\mu_{UY} = (0.4, 0.9)$, $\sigma_Y^2 = 0.3$
\end{itemize}

The validity of relevant assumptions imposed in Section \ref{sec:apps} follows immediately by the temporal order of data generating process and the underlying independence. The simulation results are presented in Figure \ref{fig:syntheticexp2}. Once again, all three estimators attain smaller bias as sample size becomes larger and the PDR estimator beats POR and PIPW estimators in all cases.

\begin{figure}[t!]
    \centering
    \includegraphics[scale = 0.7]{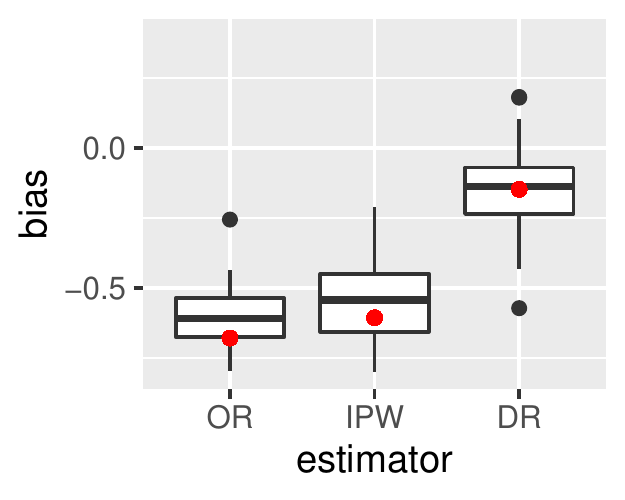}
    \includegraphics[scale = 0.7]{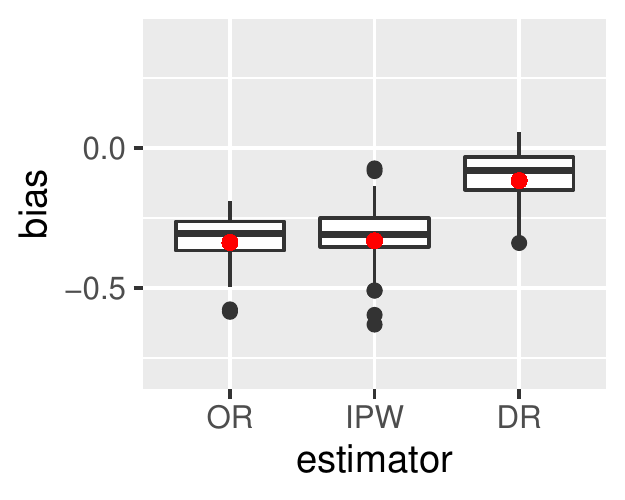}
    \includegraphics[scale = 0.7]{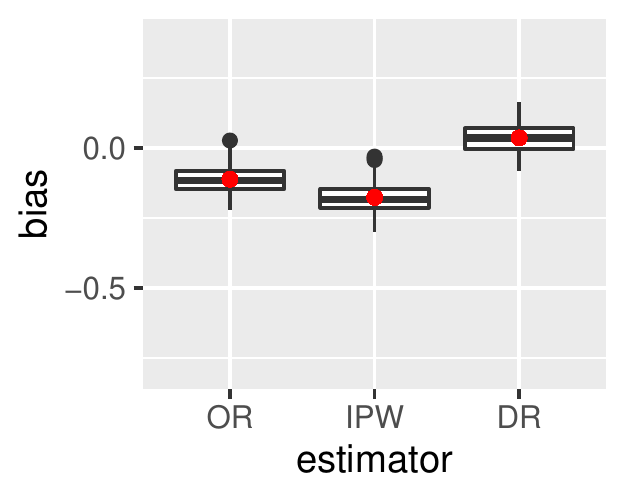}
    \caption{Boxplots of simulation results when $N = 400, 800, 1600$, respectively.
    }
    \label{fig:syntheticexp2}
\end{figure}

\subsection{Nonlinear Confounding Bridge Functions}
In this subsection we consider a simulating process such that the confounding bridge functions $h$ and $q$ are nonlinear. In particular, we add cubic terms of $(X, U)$ when generating $A$ and also interaction terms when generating $Y$. The data are generated by
\begin{equation}
    U \sim \text{Normal}(0, 1),
\end{equation}
\begin{equation}
    X_1 \sim \text{Normal}(0, 1),
\end{equation}
\begin{equation}
    X_2 \sim \text{Normal}(0, 1),
\end{equation}
\begin{equation}
    A|X, U \sim \text{Bernoulli}(1/(1 + \exp(-0.25 -0.2 X_1 -0.3 X_2 - 0.1 X_1^3 - 0.05 X_2^3 - 0.25 U + 0.1U^3))
\end{equation}
\begin{equation}
    Z|A, X, U \sim \text{Normal}(0.5 + 0.5 A + 0.2 X_1 - 0.2 X_2 + 0.75 U, 1)
\end{equation}
\begin{equation}
    W|X, U \sim \text{Normal}(0.3 + 0.35 X_1 + 0.25X_2 - 0.75 U, 1)
\end{equation}
\begin{equation}
\begin{aligned}
    &Y|A, X, U \sim 
    \text{Normal}(-0.5 + A + 0.25X_1 - 0.2 X_2 - 0.5AX_1 + 0.3AX_2 \\
    &\qquad\qquad\qquad- 0.025X_1^3 + 0.03X_2^3 - X - 0.3U + 0.25 AU + 0.025 U^3, 1)
\end{aligned}
\end{equation}

The validity of relevant assumptions imposed in Section \ref{sec:apps} follows immediately by the temporal order of data generating process and the underlying independence. The simulation results are presented in Figure \ref{fig:syntheticexp3}. Once again, all three estimators attain smaller bias as sample size becomes larger and the PDR estimator beats POR and PIPW estimators in all cases.

\begin{figure}[H]
    \centering
    \includegraphics[scale = 0.7]{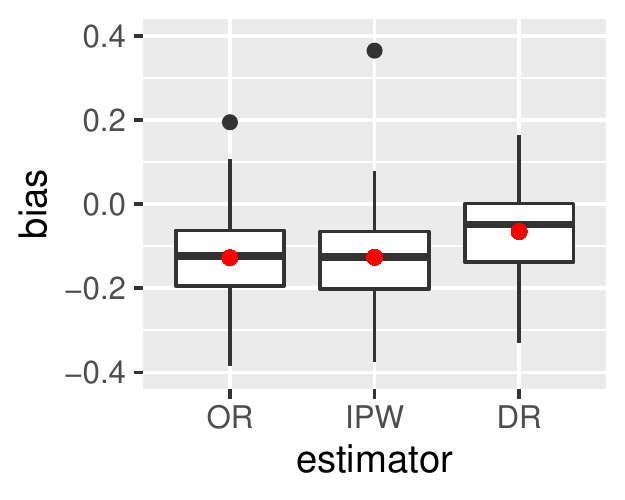}
    \includegraphics[scale = 0.7]{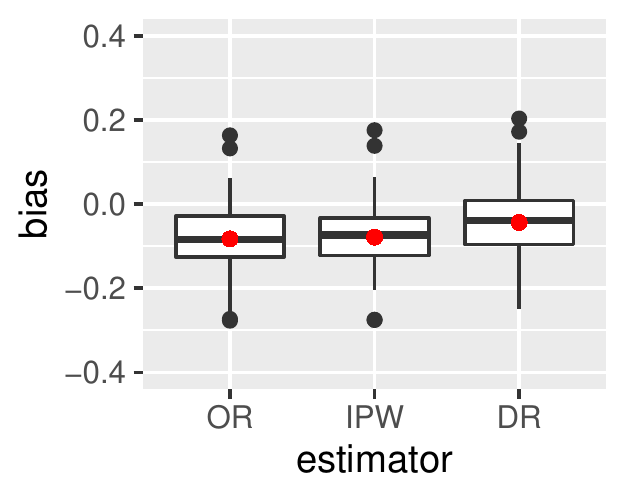}
    \includegraphics[scale = 0.7]{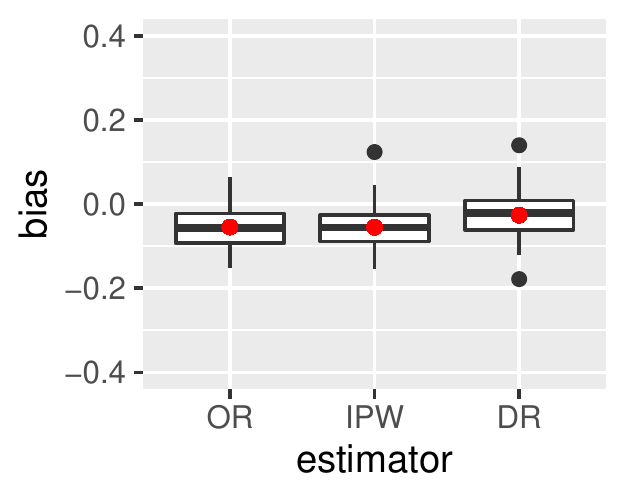}
    \caption{Boxplots of simulation results when $N = 400, 800, 1600$, respectively.
    }
    \label{fig:syntheticexp3}
\end{figure}

{\bf Real-Data Analysis.} 
We allocated $Z$ = (pafi1, paco21) and $W$ = (ph1, hema1) as the same choices of treatment- and outcome-inducing confounding proxies in \citep{tchetgen2020introduction, cui2020semiparametric}. 
The rest baseline covariates are $X$ = (age, sex, cat1\_coma, cat2\_coma, dnr1, surv2md1, aps1).
We chose hyperparameters $\lambda^{h}_{\Q} = 0.01$, $\lambda^{q}_{\sH} = 0.01$, $\lambda^{h}_{\sH} = 0.001$ and $\lambda^{q}_{\Q} = 0.001$. We set the kernel bandwidths 35 and 20 in $K_{\sH,n}$ and $K_{\Q,n}$.

\end{document}